\documentclass[10pt]{article}

\usepackage{arxiv}

 \usepackage{amsmath}
\usepackage{amssymb}
\usepackage{amsthm}
\usepackage{color}
\usepackage{enumitem}
  
\usepackage{float}
\usepackage{tikz}
\usepackage{listings}
 \usepackage{hyperref}

\usepackage{thmtools}
\usepackage{thm-restate}
\usepackage{cleveref}

\usepackage{xcolor}
\newtheorem{theorem}{Theorem} 
\newtheorem{proposition}[theorem]{Proposition}
\newtheorem{definition}[theorem]{Definition}
\newtheorem{remark}[theorem]{Remark}

\newtheorem{lemma}[theorem]{Lemma}
\newtheorem{corollary}[theorem]{Corollary}

\usepackage{algorithm}
\usepackage{algpseudocode} 
  
 \usepackage{mathrsfs}
\DeclareMathAlphabet{\mathpzc}{OT1}{pzc}{m}{it}
 \usepackage[mathscr]{eucal}

 \usepackage{xcolor}

\DeclareMathOperator*{\argmin}{argmin}

\DeclareMathOperator{\Var}{Var}

\DeclareMathOperator{\cX}{\mathscr{X}}

\DeclareMathOperator{\cF}{F}

\DeclareMathOperator{\cD}{\mathscr{D}}
\DeclareMathOperator{\cT}{\mathscr{T}}
\DeclareMathOperator{\cA}{\mathscr{A}}

\DeclareMathOperator{\cz}{\mathpzc{z}\,}
 \DeclareMathOperator{\RR}{\mathbb{R}}
\DeclareMathOperator{\NN}{\mathbb{N}}

 \DeclareMathOperator{\Unif}{Unif}

\DeclareMathOperator{\Rad}{\text{Rad}}

 \newcommand{\eqdef}{:=}

\newcommand{\norm}[1]{\left\| #1 \right \|}
\newcommand{\dprod}[1]{\left\langle #1 \right\rangle}

\newcommand{\Prob}[1]{~\mathbb{P}\left[ #1 \right ]}
\newcommand{\Exp}[1]{~\mathbb{E}\left[ #1 \right ]}
\newcommand{\ExpD}[2]{~\mathbb{E}_{#1}\left[ #2 \right ]}

\newcommand{\tExpD}[2]{~\mathbb{E}_{#1}[ #2 ]}
\newcommand{\tExp}[1]{~\mathbb{E}[ #1  ]}
\newcommand{\tnorm}[1]{\| #1  \|}
\newcommand{\absv}[1]{\left| #1  \right|}

\newcommand{\E}{\mathbb E}

\newcommand{\fF}[1]{F^{#1}}
\newcommand{\nFf}[1]{\nabla F^{#1}}

\DeclareMathOperator{\dD}{\mathfrak{D}}
 \newcommand{\iid}{\scriptstyle{\mathrm{iid}}}

\renewcommand{\epsilon}{\varepsilon}
\renewcommand{\leq}{\leqslant}
\renewcommand{\geq}{\geqslant}

\newcommand{\shorttitle}{Improved Regret Bounds via Smoothness}
\title{Between Stochastic and Adversarial Online Convex Optimization: Improved Regret Bounds via Smoothness}
\usepackage{times}
 \author{
 Sarah Sachs\\
 University of Amsterdam\\
Korteweg-de Vries Institute for Mathematics\\
 \texttt{s.c.sachs@uva.nl}
 \And 
 H\'edi Hadiji \\
  University of Amsterdam\\
Korteweg-de Vries Institute for Mathematics\\
 \texttt{hedi.hadiji@gmail.com}
 \And
  Tim van Erven\\
   University of Amsterdam\\
Korteweg-de Vries Institute for Mathematics\\
 \texttt{tim@timvanerven.nl}
   \And
    Crist\'obal Guzm\'an\\
    University of Twente\\
    Department of Applied Mathematics\\
    Pontificia Universidad Cat\'olica de Chile\\
     Institute for Mathematical and Computational Eng.\\
      \texttt{c.a.guzmanparedes@utwente.nl}
 }
\date{7. June 2022}
\begin{document}

\maketitle

\begin{abstract}
  Stochastic and adversarial data are two widely studied settings in online learning. But many optimization
tasks are neither i.i.d.~nor fully adversarial, which makes it of  fundamental interest to get a better theoretical understanding of the world between these extremes.
 In this work we establish novel regret bounds for online convex
 optimization in a setting that interpolates between stochastic
 i.i.d.~and fully adversarial losses. By exploiting smoothness of
 the expected losses, these bounds replace a dependence on the maximum
 gradient length by the variance of the gradients, which was previously
 known only for linear losses. In addition, they weaken the i.i.d.\
 assumption by allowing, for example, adversarially poisoned rounds,
 which were
 previously
 considered in the expert and bandit setting. Our results extend this to the online convex optimization framework.  
  In the fully i.i.d.\ case, our bounds match the rates one would expect
 from results in stochastic acceleration, and in the fully adversarial
 case they gracefully deteriorate to match the minimax regret.    
We further provide lower bounds showing that our regret upper bounds are
tight for all intermediate regimes in terms of the stochastic variance and the
adversarial variation of the loss gradients.
\end{abstract}

 \section{Introduction}
  
 Two of the main approaches for solving convex optimization problems under uncertain data are stochastic
 convex optimization (SCO) \cite{doi:10.1137/1027074,10.5555/2678054}
 and online convex optimization (OCO) \cite{zinkevich2003online}.
 These two models are very different in their assumptions and goals,
 despite the fact that they share many techniques. In SCO it is assumed
 that the
 loss functions follow an independent, identically distributed
 (i.i.d.)~process, and the goal is to minimize the {\em excess risk}, which is the optimization error under the expected loss. By contrast, in OCO the losses can be choosen adversarially and the goal is to minimize the {\em cumulative regret}, which is the difference between the cumulative incurred losses over rounds against the best fixed strategy in hindsight.
 Much less is known about what happens in between, in scenarios that
 interpolate between the i.i.d.~and adversarial settings. This intermediate setting has drawn major attention in the recent years in the expert and bandit setting \cite{ito2021on-optimal} \cite{amir2020prediction} \cite{pmlr-v89-zimmert19a}, however, as mentioned in \cite{ito2021on-optimal},  little is known for online convex optimization. 
 Our work studies this
 in a generalization of the OCO setting, in which  nature
 chooses
 distributions for the data that may vary arbitrarily over time, and we
 provide regret bounds in terms of two quantities that measure 
how adversarially these distributions are. The standard OCO setting
 corresponds to the case where the distributions are point-masses on
 adversarial data points.

 \paragraph{Main Contribution.} Our main contribution is a new analysis
 of {\em optimistic online algorithms} \cite{rakhlin2013online,10.5555/2999792.2999954} 
 that takes advantage of smoothness of the expected loss. This analysis
 allows for a gradual interpolation between
 worst-case adversarial regret bounds and the  best known expected regret
 bounds in the stochastic case, and also provides quantifiable
 improvements for intermediate cases.\footnote{It is well-known that in
 the fully adversarial case smoothness does not yield asymptotic
 improvements on regret \cite{HazanOCO:2016}, whereas for SCO
 improvements can be obtained only under low-noise
 \cite{ghadimi2012optimal}.}
 To capture the full range between  i.i.d.~and fully adversarial
 settings, we consider a similar adversarial model as in
 \cite{rakhlin2011online}, i.e., nature chooses distributions $\cD_t$ in
 iteration $t$, and the learner suffers loss $f_t(x_t, \xi_t)$ with
 $\xi_t \sim \cD_t$. Importantly, we do not assume that the
 distributions $\cD_t$ are all the same, but they may vary
 adversarially over time.
 
  To properly quantify the interpolation between the i.i.d.~and fully
 adversarial settings in the regret bound, we introduce two parameters
 for the loss sequence. Namely the \emph{cumulative variance}, $
 \overline{\sigma}_T  $, which captures the stochastic aspect of the
 learning task, i.e., the variance of  the $\cD_t$; and the
 \emph{cumulative adversarial variation}, $\overline{\Sigma}_T$, which
 captures the adversarial difficulties of the data, i.e., the difference
 between $\tExpD{\xi \sim \cD_t}{\nabla f(\, \cdot\,,\xi)}$ and
 $\tExpD{\xi \sim \cD_{t-1}}{\nabla f(\, \cdot\,,\xi)}$. With these two
 key quantities, our first main result in Theorem~\ref{thm:VarStepsize}
 shows that the expected regret $ \tExp{R_T(u)}$, that is the difference of the cumulative losses of the learner and a fixed solution in hindsight, is bounded by
  \begin{equation}\label{eqn:generalconvexmainresult}
 \Exp{R_T(u)} = O\!\left(D(\bar \sigma_T + \bar \Sigma_T)\sqrt{T} + LD^2\right),
 \end{equation}
  where $L$ is the smoothness constant of the expected functions
  $\fF{t} = \ExpD{\xi \sim \cD_t}{f(\, \cdot\,,\xi)}$. If, in addition, the functions
  $\fF{t}$ are $\mu$-strongly convex, then in Theorem \ref{theorem:scRegretbound} we obtain
 \begin{align*}
 \E[R_T(u)] = O\!\left( \frac{1}{\mu}
 \left(\sigma^2_{\max}+ \Sigma^2_{\max} \right) \log T + LD^2 \kappa \log \kappa\right).
 \end{align*}
 Both bounds are tight: we prove 
 matching lower bounds in Theorems \ref{theorem:regretLsmoothLB} and \ref{sc:lowerbound}.  
  In Section \ref{sec:convexsmooth} we show that our results match the
  known adversarial regret bounds as well as the best results in the
  i.i.d.\ case. For the latter, only the linear case was so far obtained
  directly via a regret analysis (see Sec.~5.2 in
  \cite{rakhlin2013online}, and prior work \cite{Hazan2010ExtractingCF,pmlr-v23-chiang12}). Using optimistic mirror descent, they
  obtained the regret guarantee of
  $R_T(u) \leq O(\sqrt{\sum_{t=1}^T\tnorm{\nabla f(x_t,\xi_t) -
  m_t}^2})$, where $m_t$ denotes an optimistic guess of the gradient
  that is chosen before round $t$.
  In the i.i.d.\ case with the prediction $m_t = \nabla f(x_{t-1},\xi_{t-1})$, this can be shown to imply
  that the expected regret is upper bounded  by $\E[R_T(u)] = O(\sigma \sqrt{T} +
  \sqrt{\sum_{t=1}^T \tExp{\tnorm{\nFf{t}(x_t) -
  \nFf{t}(x_{t-1})}^2}})$, where $\sigma$ denotes the variance of the
  stochastic gradients. This simplifies to
  \[
  \E[R_T(u)] = O(\sigma\sqrt{T})
  \qquad \text{for the i.i.d. case with linear functions $\nFf{t}$,} 
  \]
  which is a special case of \eqref{eqn:generalconvexmainresult},
  because $\sigma = \bar \sigma_T$ and $\bar \Sigma_T = 0$ for i.i.d.\
  losses, and
  $\nFf{t}(x_t) =
  \nFf{t}(x_{t-1})$ and $L=0$ for linear functions.
  It is not immediately obvious how to generalize this
  result to general convex functions with smoothness $L > 0$, however.
  In this case, we can guess the appropriate regret bound based on 
  known convergence results for \emph{stochastic
  accelerated gradient descent} (SAGD)
  \cite{AccStochOptDM,ghadimi2012optimal}: if we knew in advance
  that the losses would be i.i.d.\ and we did not care about
  computational efficiency, then we could run a new instance of SAGD for
  each round $t$. Summing the
  known rate for SAGD over $t$ then gives
   $\tExp{R_T(u)} \leq O(\sigma \sqrt{T} + L)$ (for more details see the batch-to-online conversion in  Appendix~\ref{appendixBatchToOnline}). 
  This raises
  the question if a similar bound can also be obtained directly via a regret
  analysis, without assuming i.i.d.\ observations in advance. This
  question is then answered by our result
  \eqref{eqn:generalconvexmainresult}, which indeed reduces to this rate
  for general convex smooth i.i.d.\ functions, matching the aforementioned bound up to constants. We achieve this by using smoothness to bound $ \tExp{\tnorm{\nFf{t}(x_t) -
  \nFf{t}(x_{t-1})}^2} \leq L \tExp{\tnorm{x_t - x_{t-1}}^2}$, which can
  be canceled by a negative quadratic term that we obtain from an
  improved analysis of the regret. 
   The use of this negative term in the analysis dates back to
   \cite{Nemirovski:2004}, who used it to achieve an improved $O(1/T)$ rate on the extra-gradient method.  
   
    In addition to unifying the analysis of these two extreme cases and obtaining the best known results via one algorithm (OFTRL \eqref{Algo:convex} for convex functions and OFTRL on a surrogate loss \eqref{Algo:sc} for strongly convex functions), our results give a new insight for intermediate cases. Thus, as a second main contribution we 
    shed light on a setting which is neither fully adversarial nor i.i.d.. To illustrate this, we highlight some examples here, which  received attention in the recent literature. 
  \paragraph{Adversarial corruptions: } Consider i.i.d.\ functions with
  adversarial corruptions, as considered in the context of the expert
  and bandit settings in \cite{ito2021on-optimal},
  \cite{amir2020prediction}. If the (cumulative) corruption level is
  bounded  by a constant $C$, in \cite{ito2021on-optimal}  an expected
  regret bound of $\Exp{R_T(u)} = O(R^s_T + \sqrt{C R^s_T})$ was obtained, where
  $R^s_T$ denotes the regret with respect to the uncorrupted data.  In
  \cite{ito2021on-optimal}, the authors raised the question of whether it is possible to obtain regret bounds with a similar square-root dependence on the corruption level $C$ for online convex optimization. Indeed, for this 
   intermediate model, we derive a regret bound
  \begin{align*}
 \Exp{R_T(u)} \leq O(R_T^s + D\sqrt{GC}),
  \end{align*}
   for the general convex case from our Theorem \ref{thm:VarStepsize}. We elaborate on this in Section~\ref{sec:perturbed}.  
  
  \paragraph{Random order models} The random order model (ROM) dates
  back to \cite{Kenyon97best-fitbin-packing} in combinatorial online
  learning. It has drawn attention in the online convex optimization community as an elegant relaxation of the adversarial model \cite{sherman2021optimal,pmlr-v119-garber20a}. 
  Complementary to the results in \cite{sherman2021optimal}, we show that the dependence on $G$ in the regret bound can be reduced to a dependence of $\sigma$, where $\sigma$ denotes the variance of gradients in the uniform distribution over 
  loss functions $f_1, \dots f_T$. That is,
  \begin{align*}
  \Exp{R_T(u)} \leq O\bigg(D \sigma \sqrt{T \log\Big( e \frac{\widetilde \sigma}{\sigma}\Big) } \, \bigg),
  \end{align*}
  where $\widetilde\sigma$ denotes a slightly weaker notion of variance (see Corollary \ref{cor:rom}).
  We derive these results from our main theorem under  stronger
  assumptions than in \cite{sherman2021optimal}, but we also obtain a
  better rate with $\sigma$ instead of $G$ as the leading factor, so the
  results are not directly comparable.
 We also consider a variant of the random order model, which we call the \emph{multiple pass random order model (multi-pass ROM)}. This is inspired by multiple shuffle SGD and can be considered another intermediate example between adversarial and stochastic data. We elaborate on both examples, i.e., the ROM and  multi-pass ROM in Section~\ref{sec:ROM}.

 \subsection{Related work}
  As mentioned in the previous section, our work is inspired by results in the \emph{gradual variation} and in the \emph{stochastic approximation} literature. 
  The gradual variation literature dates back to
 \cite{Hazan2010ExtractingCF}, with
 later extensions by   \cite{pmlr-v23-chiang12} and
 \cite{rakhlin2013online,10.5555/2999792.2999954}. 
 In addition to some technical relation to the aforementioned work, there is also a natural relation between our parameters $\bar \sigma_T$ and $\bar \Sigma_T$ to variational parameters in \cite{Yang_2013}, \cite{Hazan2010ExtractingCF} or \cite{pmlr-v23-chiang12}. However, as we elaborate in Remark \ref{rem:param}, there are some fundamental differences between these variational parameters and $\bar \sigma_T, \bar \Sigma_T$, which prevent us from directly obtaining a smooth interpolation from these results. It is also interesting to note that there is some relation between $\bar \Sigma_T$ and the \emph{path length} parameters considered in \emph{dynamic regret} bounds   \cite{zhao2021adaptivity,NEURIPS2020_93931410}. However, since their analysis targets a fundamentally different notion of regret, namely the  dynamic regret, the results are incomparable. 
 
With respect to the results, our findings are fundamentally different from the stochastic approximation literature, since we do not rely on the assumption that the data is following a distribution. However,  we were inspired by analysis techniques and the convergence thresholds set by  this literature.  
Our approach of obtaining accelerated rates by negative terms arising
 from smoothness in a regret bound has previously been used in the
 context of variational inequalities and saddle-point
 problems. Using this idea, \cite{Nemirovski:2004} obtained improved
 rates $O(1/T)$ for the extra-gradient method.  More recently
 \cite{AccStochOptDM} showed that acceleration in stochastic convex
 optimization  can benefit by negative terms arising in optimistic FTRL
 via an anytime-online-to-batch conversion \cite{Cutkosky2019AnytimeOC}.
 Although an important inspiration for our approach, the techniques of
 \cite{AccStochOptDM} do not directly carry over, because they evaluate
 gradients at the time-average of the algorithm's iterates, making them
 much more stable than the last iterate, which comes up when controlling
 the regret.  
  Algorithms used both in SCO and OCO follow a vast literature on stochastic approximation methods, e.g.~\cite{RobbinsMonro:1951, doi:10.1137/1027074, Polyak:1992}. For this work, we are particularly interested in
   the more recent literature on acceleration in SCO 
 \cite{ghadimi2012optimal, jain2018accelerating, AccStochOptDM}. In this research field 
 efficiency is traditionally measured in terms of excess risk. On the one hand, regret upper bounds can be converted into excess 
 risk bounds, through the so-called {\em online-to-batch conversions} \cite{CesaBianchi:2002}; on the other hand,
 excess risk guarantees do not directly lead to regret bounds, and even if they do some key features of the rates might be
 lost. These latter methods, known as {\em batch-to-online conversions} are discussed in Appendix \ref{appendixBatchToOnline}.

  \paragraph{Outline}
 In Section~\ref{definitions}, after setting up notation and basic definitions, we introduce the \emph{stochastically extended adversarial} model, a generalization of the standard adversarial model similar to the model used in \emph{smoothed analysis}. Our main results can be found in Section~\ref{sec:mainResults}. In Section~\ref{sec:implications} we illustrate our results by highlighting several special cases, such as the random order model and the adversarially corrupted stochastic model.
 Finally, in Section~\ref{Conclusion} we set our findings into a broader context and give perspective for future work.  
   
 \section{Setting}
 \label{sec:prelim}
 We recall the \emph{online convex optimization (OCO)} problem. Here, we consider a sequence of convex functions $f_1, \dots f_T$ defined over a closed and bounded convex set $\cX\subseteq \mathbb{R}^d$, which become available to the learner sequentially.
  In the standard \emph{adversarial model}, the learner chooses $x_t \in \cX$ in round $t$, then function $f_t$ is revealed and the learner suffers \emph{loss} $f_t(x_t)$. The success of the learner is measured against all fixed $u \in \cX$.   Hence, the goal of the learner is to minimize the \emph{regret}, that is, the difference between their cumulative loss $\sum_{t=1}^Tf_t(x_t)$ and that of the best fixed choice in hindsight, namely $\min_{u\in \cX}\sum_{t=1}^T f_t(u)$.
 
 Throughout the paper we use the notation $[T] =\{1,\dots,T\}$. We follow the notation convention that  $\delta_c$ denotes a Dirac measure at a point $c$, and $\|\cdot\|$ denotes the euclidean norm. 
  \label{definitions}
 \subsection{Stochastically Extended Adversarial Model}
 \label{def}
 We extend the aforementioned adversarial   
 model by letting  nature choose a distribution $\cD_t$ from a set of distributions. Then the learner suffers loss $f(x_t, \xi_t)$ where $\xi_t \sim \cD_t$. 
 Note that if the set of distributions is sufficiently rich, this model contains the standard adversarial model  and the stochastic model as a special cases (see Examples \ref{Adv},\ref{Stoch}).
 We introduce some notation to make this more precise. 
 Let $\cX \subset \RR^d$ be a closed convex set and $\Xi$ a measurable space. 
 Define   $f:\cX \times \Xi \rightarrow \RR$ and assume $f(\cdot,\xi)$ is convex  over $\xi \in\Xi$. Suppose $\dD$ is a set of probability distributions over $\Xi$.
 For any $\cD \in \dD$, we denote the gradient mean  by $\nFf{\cD}(x) \eqdef \ExpD{\xi \sim \cD}{\nabla f(x,\xi)}$ and the function mean $\fF{\cD}(x) \eqdef \ExpD{\xi \sim \cD}{  f(x,\xi)}$.
Furthermore, denote by $\sigma^2_{\cD}$ an upper bound on the variance of the gradients
 \begin{align*}
  \sigma^2_{\cD}  = \max_{x \in \cX}  \ExpD{\xi \sim \cD}{ \norm{\nabla f(x,\xi) - \nFf{\cD}(x)}^2}  .
 \end{align*}
 We introduce some shorthand notation when distributions are indexed by rounds. Given $t\in [T]$, we write $\fF{t}$ and $\sigma^2_t$ instead of $\fF{\cD_t}(x)$ and $\sigma_{\cD_t}^2$, respectively.
Let us now introduce the \emph{stochastically extended adversary protocol}. 
\begin{definition}[Stochastically Extended Adversary (SEA)]
  In each round $t$, the learner chooses $x_t \in \cX$, the SEA picks $\cD_{t} \in \dD$. The learner and the SEA both observe a sample $\xi_t \sim \cD_{t}$,  and the learner suffers loss $f(x_t,\xi_t)$.  
 \end{definition}
 Note that the SEA model is closely related to the adversarial model considered in the context of \emph{smoothed analysis} \cite{rakhlin2011online, haghtalab2022smoothed, spielman2004smoothed}. However, in contrast to this line of work, we do not focus our attention to SEA distributions with sufficient anti-concentration (c.f., Def. 1.1 in  \cite{haghtalab2022smoothed}). Indeed, this restriction would exclude, among others, the fully adversarial case as described below. Note also that we assume that SEA has access to the realization $\xi_t$, hence, can choose  distribution $\cD_{t+1}$ based on $\xi_1, \dots,\xi_t$. This assumption is not relevant for the fully adversarial nor the i.i.d.~setting. In the former, because there is no randomness, and in the latter because there is no change in distribution. However, it is relevant for some of the intermediate cases, and in particular in the random order model.   The SEA model contains several common settings from the literature as special cases. To illustrate this, we list some examples. 
  \begin{enumerate}
  \item \textbf{Adversarial Model:} \label{Adv}The SEA chooses a Dirac measure $\delta_{c_t} \in \dD$ in each round. Then for any $\xi_t \sim \delta_{c_t}$, the SEA selects $f(\, \cdot\, , \xi_t)$, and the model reduces to an adversary selecting directly the functions $ f_t(\, \cdot\,)$. 
   \item \textbf{Stochastic I.I.D.~Model:}\label{Stoch} The SEA chooses a fixed $\cD \in\dD$ and selects $\cD_{t} = \cD$ at each round $t$.  
   \item \textbf{Adversarially Corrupted i.i.d.~Model:} The adversary selects an i.i.d.~source $\mathcal D$ and perturbs the data with adversarial corruptions. This fits in our framework by considering that, given a corruption level $C\geq 0$, the SEA chooses distributions $\cD_t = \cD \otimes \delta_{c_t} $ such that $\sum_{t=1}^T \sup_{x \in \cX}\norm{\ExpD{\xi \sim \cD}{\nabla f(x,\xi)} - \ExpD{\xi' \sim \cD_t}{\nabla f(x,\xi')}}\leq C$. 
   \item \textbf{Random Order Models (ROM): } 
   Among a fixed family of losses $\mathcal F = (f_i, \, i \in [n])$, the SEA randomly picks functions in $\mathcal F$ via sampling without replacement, possibly performing multiple passes over the losses and reshuffling between the passes.
   Formally, define $\Xi = [n]$, and $\xi_t \in \Xi$ to be the $t$-th loss pick; if $t \in [nk, n(k+1)]$ for some $k \in \NN$, then the SEA chooses the distribution $\mathcal D_t = \Unif( \Xi \setminus \{ \xi_s : \, s \in [nk+1, n(t-1)]  \} )$. 
   \end{enumerate} 
 To quantify the hardness of the loss sequence, we introduce the \emph{cumulative stochastic variance} and \emph{adversarial variation}; we also define an average of these quantities.   
 We denote by $\E$ the expectation taken with respect to the joint distribution of $(x_1, \xi_1, \dots, x_T, \xi_T)$. Note that the choice of the adversary $\mathcal D_t$ can be random itself, as it depends on the past observations (of both the player's actions and the realizations of the $\xi_t$'s). In this case, $\sigma_t$ is also a random quantity.
  
 \begin{definition}[Cumulative Stochastic Variance and Cumulative Adversarial Variation]
  Suppose the SEA chooses distributions $\cD_{1}, \dots, \cD_T$. Recall that $\sigma_t^2$ is a shorthand for $\sigma^2_{\cD_t}$. The cumulative stochastic variance and the  cumulative adversarial variance are defined as
  \begin{align*}
    \sigma^{(2)}_{[1:T]} = \E \bigg[\sum_{t=1}^T \sigma^2_t \bigg]
    \qquad \text{ and } \qquad
    \Sigma^{(2)}_{[1:T]}=  \E \bigg[\sum_{t=1}^T  \sup_{x\in\cX} \norm{\nFf{t}(x) - \nFf{t-1}(x)}^2 \bigg].
  \end{align*}
 We also let $\bar \sigma_T$ and $\bar \Sigma_T$ denote the square root of the average stochastic variance or adversarial variation, respectively; that is,
$
\bar \sigma_T^2 = \sigma_{[1:T]}^{(2)} / T
$
and 
$
\bar \Sigma_T^2 = \Sigma_{[1:T]}^{(2)} / T.
$
\end{definition}
 Note that in the special case when all $f_t$ are fully adversarial, $\bar \sigma_T = 0$. On the contrary, in the stochastic case, i.e., if all for each round $t$, the distribution $\cD_t$ is equal to a fixed (but arbitrarily chosen) $\cD$, then $\bar \Sigma_T = 0$. In this case, $\bar \sigma_T  $ reduces to the common definition of the gradient variance upper bound in the SCO literature \cite{Ghadimi2013StochasticFA,ghadimi2012optimal}. 
 If however, the SEA chooses one distribution $\cD_i$ for the first rounds and then switches to a different distribution $\cD_j$, then $\bar \sigma_T$ can only be upper bounded by $\max(\sigma_i, \sigma_j)$. This upper bound can be pessimistic, however, for some results it gives a better intuition. For this purpose we also define the \emph{maximal stochastic variance} and \emph{maximal adversarial variation}. 
 \begin{definition}[Maximal Stochastic Variance and Maximal Adversarial Variation]  
 Let $\sigma^2_{\max}$ be an upper bound on all variances $\sigma^{2}_{t}$. That is,
 \begin{align*}
 \sigma^2_{\max} = \max_{t\in[T]}  \, \E \big[ \sigma^{2}_{t} \big]
  \quad \text{and} \quad
 \Sigma^2_{\max} = \max_{t\in[T]} \, \E \bigg[\sup_{x\in\cX}\norm{\nFf{t}(x) - \nFf{{t-1}}(x)}^2\bigg] .
 \end{align*}
 \end{definition}
 
 \begin{remark}
 \label{rem:param}
  As we mentioned in the introduction, the cumulative stochastic variance and the adversarial variation have some similarities with parameters in gradual variation regret bounds. 
  For linear functions $\langle \mu_t, \cdot \rangle$, the bounds in \cite{Hazan2010ExtractingCF} involve the parameter $\smash{\Var_T = \sum_{t=1}^T\norm{ \mu_t - \bar \mu_T}^2}$ where $ \bar \mu_T$ is the average of the gradients.    For OCO with general convex functions, \cite{pmlr-v23-chiang12} provide upper bounds on the regret in terms of the $L_p$-deviation $D_p = \smash{\sum_{t=1}^T \sup_{x\in\cX}\tnorm{\nabla f_t(x) - \nabla f_{t-1}(x)}^2_p }$.    In Lemmas~\ref{lem:example:VarT} and~\ref{lem:Dp:comp} in Appendix~\ref{appendix:sec:Lsmooth}, we show that in the SEA framework, both of these types of bounds are generally worse than ours, and that the difference can be arbitrarily large.
 In \cite{10.5555/2999792.2999954}  the regret is bounded in terms of  $\smash{\sum_{t=1}^T \tnorm{g_t - M_t}^2}$.  As mentioned in the introduction, unless the loss functions are linear or the learner has knowledge of the gradient mean, $\sum_{t=1}^T \tnorm{g_t - M_t}^2$ cannot directly be reduced to $\smash{\sigma_{[1:T]}^{(2)}}$ or $\smash{\Sigma_{[1:T]}^{(2)}}$.

  \end{remark}
    
 \subsection{Assumptions}
 In our analysis we will frequently use several of the following additional assumptions. Some of these were already mentioned in the introduction. We keep them all together here, for the convenience of the reader and clear reference.  For any $\mathcal D \in \dD$:
 \begin{enumerate}[label=\textbf{(A\arabic*) }]
 \addtocounter{enumi}{-1}
 \item \label{A0}   
 the adversary has access to independent samples $\xi  \sim \cD$.  
 \item\label{A1}  
 the function $f(\, \cdot\,, \xi )$ is convex, and gradients are bounded by $G$ a.s. when $\xi\sim \mathcal D$.
 \item\label{A2}   
 the expected function $\smash{\fF{\cD}}$ is $L$-smooth, i.e, $\smash{\nFf{\cD}}$ is $L$-Lipschitz continuous.
 \item\label{A3}  
 for any $x\in\cX$, the variance $\tExpD{\xi \sim \cD}{ \tnorm{\nabla f(x,\xi) - \nFf{\cD}(x)}^2}$ is finite.  
  \item\label{A4}    
 the expected function \smash{$\fF{\cD}(\, \cdot\, )$} is $\mu$-strongly convex.
 \end{enumerate}
We assume that \ref{A0} always holds. 
Assumptions \ref{A1},\ref{A2} and \ref{A3} are standard in stochastic optimization, and are similar to common assumptions for online convex optimization. There, it is typically assumed that the adversarial samples $f_t(\,\cdot\,)$ are convex (or even linear) and the gradient norms $\norm{\nabla f_t(\cdot)}$ are bounded.   
 Note that we only require gradient Lipschitz continuity and strong convexity to hold for the expected loss.

 \section{Algorithms and Regret Bounds} \label{sec:mainResults}
 \subsection{Convex Smooth Functions}
 \label{sec:convexsmooth}
 We use \emph{Optimistic Follow-the-Regularised-Leader} (OFTRL) (see, e.g., \cite{joulani2017modular, rakhlin2013online}) to minimize regret. Let $\smash{(\eta_t)_{t\in[T]}}$ be a non-decreasing and positive sequence of stepsizes, possibly tuned adaptively with the observations. At each step $t$, the learner makes an optimistic prediction $\smash{M_t \in \RR^d}$ and updates its iterates as
 \begin{align}
 \label{Algo:convex}
   x_{t} = \argmin_{x \in \cX} \bigg\{  \bigg\langle x, M_t + \sum_{s=1}^{t-1} g_s \bigg\rangle + \frac{\|x\|^2}{\eta_t} \bigg\}, 
 \end{align}  
 where we denoted by $g_t = \nabla f(x_t, \xi_t)$ the observed gradient at time $t$. To state our results, we denote by $\!\Exp{\, \cdot\,}$ the expectation with respect to the joint distribution of $(x_1, \xi_1, \dots, x_T, \xi_T)$. Our objective is to bound the average regret:
 \begin{align*}
  \Exp{R_T(u)} \eqdef \E\bigg[\sum_{t=1}^T \dprod{g_t, x_t - u}\bigg].
 \end{align*}
 The following theorem, proved in Appendix~\ref{Appendix:VarStepsize}, is our main result for convex functions.
 \begin{restatable}{theorem}{thmVarStepsize}\label{thm:VarStepsize}
     Fix a user-specified parameter $\nu > 0$. Under assumptions \ref{A1},\ref{A2},\ref{A3}, OFTRL, with $M_t = g_{t-1}$ and adaptive step-size $\eta_t = D^2/(\nu+\sum_{s=1}^{t-1}\eta_s\norm{g_s-M_s}^2 )$, has regret
   \begin{equation}\label{eq:full_bound_non_sc}
     \Exp{R_T(u)}  		\leq
     D\big(6 \,  \bar \sigma_T   +  3\sqrt 2 \bar\Sigma_T\big)\sqrt T \, + \frac{3\sqrt 2 DG}{2}  + \nu + \frac{1}{\nu} \left( 4D^2G^2  +  9L^2D^4 \right).
   \end{equation}
   The algorithm needs only the knowledge of $D$. With the extra knowledge of $G$ and $L$, one can tune $\nu = LD^2 + DG^2$ to get
   \begin{align*}
     \Exp{R_T(u)} \leq O \big(
     D\big(   \bar \sigma_T   +   \bar\Sigma_T\big)\sqrt{ T }\, +  DG + LD^2 
     \big).
   \end{align*}
 Moreover, if only convexity of the individual losses holds \ref{A1}, then tuning $\nu = 2DG$ ensures the (deterministic) bound $R_T(u) \leq 3\sqrt{2} DG \sqrt T + 4DG $ \, .
 \end{restatable}
Without prior knowledge of the smoothness parameter, the best the player can do is to tune $\nu$ according to a guessed value $L_0$. This affects the constants in the bound by an additive term of order $ (L_0 + L^2 /L_0)D^2$; it would be interesting to determine if this is an inevitable price to pay for the lack of knowledge of $L$. A similar discussion can be held for $G$.
Note that the worst-case regret bound of order $DG\sqrt T$ always holds every time OFTRL is used in this article, even without smoothness. To avoid distraction, we will not recall this fact in the applications.

The algorithm and analysis dwell on two ideas: the adaptive tuning of the learning rate à la AdaHedge/AdaFTRL \cite{DBLP:journals/corr/abs-1301-0534, Orabona2018ScalefreeOL} with optimism, together with the fact that we keep a negative Bregman divergence term in the analysis, which is crucial to obtain our bound. 

  The upper bound in Theorem \ref{thm:VarStepsize} is tight up to additive constants, as the following result shows. 
 \begin{restatable}{theorem}{LsmoothLB}
 \label{theorem:regretLsmoothLB}
  For any learning algorithm, and for any pair of positive numbers $(\sigma, \Sigma)$ there exists a function $f:\cX \times \Xi \rightarrow \RR$ and a sequence of distributions satisfying assumptions \ref{A1}, \ref{A2},\ref{A3} with $\bar \sigma_T \geq \sigma$ and $\bar \Sigma_T \geq \Sigma $ such that
  \begin{align*}
  \Exp{R_T(u)} \geq  \Omega\big( D\left( \bar \sigma_T +   \bar\Sigma_T \right)\sqrt{ T} \big).
  \end{align*}
 \end{restatable}
 The proof, in Appendix~\ref{AppendixLBsmooth} relies on a lower bound from stochastic optimization \cite{6142067, doi:10.1137/1027074} together with the fact that we can construct a sequence of convex and $L$-smooth loss functions such that $\bar \Sigma_T$ is in the order of the gradient norms $G$. Combining these insights with the lower bound $\smash{\Omega(DG\sqrt{T})}$ \cite{Orabona2018ScalefreeOL} gives the desired result. 
 \subsection{Strongly Convex and Smooth Functions}
 Up to this point, we have only considered functions which satisfy the weaker set of assumptions \ref{A1},\ref{A2},\ref{A3}. In this section, we show what improvements can be achieved if strong convexity also holds, that is, if \ref{A4} is satisfied with some known parameter $\mu > 0$.  For $g_t = \nabla f(x_t,\xi_t)$, define the surrogate loss function
 \begin{equation}\label{eq:surrogate-quadratic}
   \ell_t(x) = \dprod{g_t, x-x_t} + \frac{\mu}{2}\norm{x-x_t}^2.
 \end{equation}
 We use Optimistic Follow-the-Leader (OFTL) on the surrogate losses. For each step $t$, the learner makes an optimistic prediction of the next gradient $M_t \in \mathbb R^d$ and selects
 \begin{align}
 \label{Algo:sc}
   x_{t} = \argmin_{x \in \cX} \bigg\{ \sum_{s=1}^{t-1}   \ell_s(x) + \langle M_t, x \rangle  \bigg\} \,.
 \end{align}   
 The next theorem is analogous to Theorem~\ref{thm:VarStepsize} for curved losses, and will be our main tool in establishing results for strongly convex losses; see Appendix \ref{appendix:scRegretbound} for a proof.
 \begin{restatable}{theorem}{scRegretbound} \label{theorem:scRegretbound}
 Under assumptions \ref{A1}--\ref{A4}, the expected regret of OFTL with $M_t = \nabla f(x_{t-1},\xi_{t-1})$ on surrogate loss functions $\ell_t$ defined in \eqref{eq:surrogate-quadratic} is bounded as 
 \begin{align*}
   \Exp{R_T(u)} &\leq  \frac{1}{\mu} \sum_{t=1}^T \frac{1}{t}\left(8\sigma_{\max}^2 + 4 \E \bigg[\sup_{x\in\cX}\norm{\nFf{t}(x) - \nFf{t-1}(x)}^2 \bigg]\right) +   \frac{4 D^2 L^2}{\mu} \log \left(1+  \frac{16L}{\mu} \right) \\
   &\leq\frac{1}{\mu}  \left( 8 \sigma_{\max}^2
    + 4  \Sigma_{\max}^2 \right)\log T  +   \frac{4D^2 L^2}{\mu} \log \left(1+  \frac{16L}{\mu} \right) \, .
 \end{align*}
 
 \end{restatable}
 Note that OFTL requires no tuning besides the strong convexity parameter used in the surrogate losses. In particular, it is adaptive to the smoothness $L$.  
 \paragraph{Lower Bound}
The bound in Theorem~\ref{theorem:scRegretbound} is tight, as the next result, proved in Appendix~\ref{app:sc:lowerbound} shows.
\begin{restatable}{theorem}{sclowerbound}\label{sc:lowerbound}
  For any learning algorithm, and for any pair of positive numbers $(\sigma, \Sigma)$ there exists a function $f:\cX \times \Xi \rightarrow \RR$ and a sequence of distributions satisfying assumptions \ref{A1},\ref{A2},\ref{A3} and \ref{A4} with $\sigma_{\max} \geq \sigma$ and $\Sigma_{\max} \geq \Sigma $ such that
  \begin{align*}
  \Exp{R_T(u)} \geq  \Omega \left(   \frac{1}{\mu} \left(\sigma^2_{\max}+ \Sigma^2_{\max} \right) \log T \right).
  \end{align*}
\end{restatable}  

\section{Implications}
\label{sec:implications}
We derive consequences of our results from Section~\ref{sec:mainResults}. Further examples can be found in Appendix \ref{appendix:examples:intermediate}.
\subsection{Interpolating Known Results: Fully Adversarial and i.i.d. Data}\label{sec:adv_and_iid}
A first implication of our analysis is that we recover both the adversarial and i.i.d. rates, \emph{via a single adaptive algorithm}.
\paragraph{Convex Case} For adversarial data, $\sigma_t = 0$ for all $t$, and $\smash{\Sigma_{[1:T]}^2} \leq G \sqrt{2T}$. Thus, Theorem~\ref{thm:VarStepsize} guarantees a bound  $R_T(u) \leq O( DG\sqrt{T})$, which is known to be the optimal rate up to the additive constants, cf. \cite{zinkevich2003online} (note that the expectation does not act on the regret in this case).
Simultaneously, if the data is i.i.d., then Theorem~\ref{thm:VarStepsize} guarantees that
\begin{equation}\label{regret:LsmoothStochastic}
  \Exp{R_T(u) } \leq O\big( D\sigma \sqrt{T} +  LD^2 + DG \big).
\end{equation}
From standard online-to-batch conversion, this implies an excess risk for the related SCO problem of order $O(D \sigma/ \sqrt{T} +  D(L D + G)/T)$, which matches the well-known result by \cite{Ghadimi2013StochasticFA} up to lower order terms. 
On the other hand, using batch-to-online conversion (see Appendix~\ref{appendixBatchToOnline}) with the best known accelerated convergence result in SCO, gives $O(D \sigma \smash{\sqrt{T}} + LD^2)$ regret. Therefore, up to a constant, our result coincides with the best known results from SCO.
Note that also generalizes the improvement obtained for linear functions in the i.i.d.~setting in \cite[Section 6.2]{rakhlin2013online}. 
    \paragraph{Strongly Convex Case}
  The adaptive interpolation between i.i.d.~and adversarial rates also holds in the strongly convex case. 
  For adversarial data, the bound of Theorem~\ref{theorem:scRegretbound} is of order $(G^2 / \mu) \log T$, which is known to be the optimal worst-case rate, cf. \cite{pmlr-v19-hazan11a}. 
  For i.i.d.~data, the dependence on $G^2$ improves to $\sigma^2$, yielding a bound of order $ O((\sigma^2 / \mu) \log{T} + L D^2 \kappa \log \kappa )$. 
  This improvement is akin to improvements obtained by \emph{accelerated stochastic gradient descent} in the context of stochastic optimization \cite{ghadimi2012optimal, AccStochOptDM}. In fact, applying batch-to-online conversions and summing the optimization rates would yield a regret bound similar to ours; c.f. Appendix~\ref{appendixBatchToOnline}.
  \subsection{Adversarially Corrupted Stochastic Data}\label{sec:perturbed}
  We consider a natural generalization to online convex optimization of the corruption model considered in the bandit literature \cite{pmlr-v32-seldinb14, pmlr-v89-zimmert19a}, also recently studied in \cite{ito2021on-optimal} for prediction with expert advice. There, the author obtains a regret bound that is the sum of the i.i.d.~rate and of a term of order $\sqrt C$ where $C$ is the total amount of perturbation. They then raise the open question of whether similar results could be obtained for general convex losses. We provide a positive answer to this question in this section, with the regret bound in Corollary~\ref{cor:perturb}.
  
  In this model, the generating process of the losses is decomposed as a combination of losses coming from i.i.d.~data, with a small additive adversarial perturbation. This fits in the framework by setting $
  \xi_t = (\xi_{\iid, t}, c_t) \sim \mathcal D_t = \mathcal D \otimes \delta_{c_t}
  $ and 
  \begin{equation*}
    f(x, \xi_t) 
    = h(x, \xi_{\iid, t}) + c_t(x) 
  \end{equation*}
  where $c_t$ is the adversarial part of the losses selected by the adversary, and $\xi_{\iid, t} \sim \mathcal D$ is a sequence of identically distributed random variables. 
  Note that, similarly to our inspirations \cite{ito2021on-optimal, pmlr-v32-seldinb14}, and contrary to other corruption models for prediction with expert advice \cite{amir2020prediction}, we measure the regret against the perturbed data.
  Define $F = \E_{\xi \sim \mathcal D}[h( \cdot, \xi)]$, so that
  $
    F^t(x) = F(x) + c_t(x) \, . 
  $
  The amount of perturbation is measured by a parameter $C > 0 $ bounding
  \begin{equation*}
    \sum_{t=1}^T \max_{x \in \cX}\|\nabla c_t (x)\| \leq C \, , 
  \end{equation*}
  which is a natural measure of perturbation on the feedback used by the player (note that adding a constant to the perturbations does not change the regret). 
  In this case, the adversarial perturbation on the loss does not affect the variance and $\sigma^2_{\mathcal D_t} = \sigma^2_{\mathcal D}$. The perturbation appears in the loss variation as for any $t \geq 2$, for any $x \in \cX$,
  \begin{multline*}
    \| \nabla F^t(x) - \nabla F^{t-1}(x)\|^2 
    \leq 2 G \|\nabla F^t(x) - \nabla F^{t-1}(x)\|  \\
    = 2 G \|\nabla c_t(x) - \nabla c_{t-1}(x)\|
    \leq 2G\bigl( \|\nabla c_t(x)\|+ \|\nabla c_{t-1}(x)\| \bigr) \, .
  \end{multline*}
  Upon taking the supremum over $x \in \cX$ and summing over $t$, we get (with the convention that $c_0 \equiv 0$),
  \begin{equation*}
    \Sigma_{[1:T]}^{(2)}  
    = \sum_{t=1}^T \sup_{x\in \cX} \norm{\nabla c_t(x) - \nabla c_{t-1}(x)}^2 
    \leq 4 G C \, . 
  \end{equation*}
  Hence, Theorem~\ref{thm:VarStepsize} combined with the bounds on $\bar \sigma_T$ and $\bar \Sigma_T$ yields the following regret guarantee.
  \begin{corollary}\label{cor:perturb}
    In the adversarially corrupted stochastic model, adaptive OFTRL enjoys the bound
    \begin{align*}
      \Exp{R_T(u)} = O\big(D \sigma \sqrt{T} + D  \sqrt{GC} \, \big).
      \end{align*}
  \end{corollary}
  This regret bound is the sum of the i.i.d.\ rate for the unpertured source with a term sublinear in the amount of perturbations $C$, achieved without the prior knowledge of $C$. This provides an answer to the question of \cite{ito2021on-optimal}.
  An interesting open question that remains would be to extend these results to strongly convex losses.
\subsection{Random Order Models}\label{sec:ROM}
We apply our results from Section~\ref{sec:mainResults} to the Random Order model. The online ROM was introduced by \cite{pmlr-v119-garber20a} as a way of restricting the power of the adversary in OCO. Our results highlight that the rates in the ROM model, which is not i.i.d., are almost the same as the rates of the i.i.d.\ model obtained via sampling in the same set of losses with replacement.
\begin{corollary}\label{cor:rom}
  In the single-pass ROM with convex and $L$-smooth losses $(f_k)_{k \in [T]}$, OFTRL  (c.f. \eqref{Algo:convex}) enjoys the regret bound
  \begin{equation*}
    \E[R_T(u)] \leq O \bigg( D \sigma_1  \sqrt {\log\left( e \frac{\widetilde \sigma_1}{\sigma_1}\right) T} + LD^2 + DG \bigg) \,, 
  \end{equation*}
  where 
  \begin{equation*}
    \sigma_1^2 = \max_{x \in \mathcal X} \frac{1}{T}\sum_{t=1}^T \bigg\|\nabla f_t(x) - \frac{1}{T} \sum_{s=1}^T\nabla f_s(x) \bigg\|^2 
    \; \text{and} \;
    \widetilde \sigma_1^2 =  \frac{1}{T}\sum_{t=1}^T \max_{x \in \mathcal X} \bigg\|\nabla f_t(x) - \frac{1}{T} \sum_{s=1}^T\nabla f_s(x) \bigg\|^2 \, . 
  \end{equation*}
\end{corollary}
Note that $\sigma_1 \leq \widetilde \sigma_1 \leq 4G^2$, and that the logarithm of the ratio which appears in the bound is moderate in any reasonable scenario.
 The proof of Corollary~\ref{cor:rom} consists in controlling the adversarial variation and the cumulative variance thanks to the following lemma, proved in Appendix~\ref{proof:rom}.  
\begin{restatable}{lemma}{lemROMsmooth}\label{lem:ROM:Lsmooth}
  In the single-pass ROM, we have $\Sigma_{[1:T]}^{(2)} \leq 8 G^2$ and 
  $
    \bar \sigma_{[1:T]}^{(2)} \leq T \sigma_{1}^2 \log ( 2e^2 \widetilde \sigma_1^2 /  \sigma_1^2) \, . 
  $
\end{restatable}
We would like to emphasize that our results are complementary to those of \cite{pmlr-v119-garber20a, sherman2021optimal}. The focus of these works is to relax the assumption that individual losses are convex, and to only require convexity of the average loss function, leading to very different technical challenges. Inquiring if our results can also be achieved under the weaker assumptions of \cite{sherman2021optimal} would be an interesting direction for future work.

We also consider the multi-pass ROM. Let $P \in \NN$ denote the number of passes. From Lemma \ref{lem:ROM:Lsmooth} and Corollary~\ref{cor:rom} we directly obtain \begin{align*}
  \Exp{R_T(u)} \leq \mathcal O \bigg( D \sigma_1  \sqrt {\log\Big( \smash{e\frac{\widetilde \sigma_1}{\sigma_1}}\Big) T} + DG\sqrt P + LD^2 \bigg) \, . 
\end{align*}
 
Combining Lemma~\ref{lem:ROM:Lsmooth} with Theorem~\ref{theorem:scRegretbound} also gives the following corollary for strongly convex functions; see Appendix~\ref{cor:ROM:sc:proof} for a proof.
\begin{restatable}{corollary}{corROMsc}
\label{cor:ROM:sc:mpROM}
 Under the same assumption as in Theorem \ref{theorem:scRegretbound}, the expected regret of the ROM is bounded by
 \begin{align*}
 \Exp{R_T(u)} \leq O\!\left(\frac{\sigma^2_{1}}{\mu} \log T + \frac{G^2}{\mu} + LD^2  \kappa \log \kappa \right).
 \end{align*}
 For multi-pass ROM with $P$ passes, we obtain
\begin{align*}
 \Exp{R_T(u)} \leq O\!\left(\frac{\sigma^2_{1}}{\mu} \log T + \frac{G^2}{\mu} + \frac{G^2 \log P}{n \mu} + LD^2  \kappa \log \kappa \right).
\end{align*}

\end{restatable}

  \section{Conclusion and future work}
  \label{Conclusion} 
 As we showed, exploitation of smoothness of the expected loss functions reduces the dependence of the regret bound on the maximal gradient norm to a dependence on the cumulative stochastic variance and the adversarial variation.
 Furthermore, we took a step towards a deeper theoretical understanding
 of the practically relevant intermediate scenarios. Our approach also
 opens several interesting new research directions. For instance, in the ROM, as mentioned in Section~\ref{sec:ROM}, an interesting question is whether a regret bound with dependence  on $\sigma$ instead of $G$ can also be achieved with weaker assumptions as in \cite{sherman2021optimal}.  
 Another interesting question is whether it is possible to unify the analyses and
 algorithms for the strongly convex and convex case. So far our analyses
 of these cases were intrinsically different and the  choice of the
 algorithm requires the knowledge of strong convexity constant $\mu$.
 Since this knowledge might not be available, it is of practical interest
 to design an adaptive method which can automatically get the best rate
 without manually tuning for $\mu$. 
 
 Social impact: this work is theoretical, and therefore does not entail any societal concerns.
 
 \section*{Acknowledgements}
Sachs, Hadiji and Van Erven  were supported by the Netherlands
Organization for Scientific Research (NWO) under grant number
VI.Vidi.192.095. Guzm\'an's research is partially supported by INRIA
through the INRIA Associate Teams project and the FONDECYT 1210362 project.
  
\bibliographystyle{apalike}
\bibliography{accelLit}
\newpage
\appendix

 
     \section{Proofs of Section \ref{sec:prelim}}
  \label{appendix:sec:Lsmooth}
  Note that $\Var_T = \sum_{t=1}^T\norm{\nabla f_t(x) - \mu^*_T}^2$ can be understood as an empirical approximation of $\sigma_{[1:T]}^{(2)}$. The following lemma shows the relation of $\Var_T$ to parameters $\sigma_{[1:T]}^{(2)}$ and $\Sigma_{[1:T]}^{(2)}$.
   \begin{lemma}\label{lem:example:VarT}
Define $\Var_T = \sum_{t=1}^T\norm{\nabla f_t(x) - \mu^*_T}^2$ with $ \mu^*_T =  \frac{1}{T}\sum_{t=1}^T \nabla f_t(x)$. In expectation with respect to distributions $\cD_1, \dots \cD_T$,
\begin{align*}
\Exp{\Var_T} \geq \frac{1}{5}\left(\sigma^{(2)}_{[1:T]} + \Sigma^{(2)}_{[1:T]} \right).
\end{align*} 
Furthermore, there exists distributions such that $\Exp{\Var_T}$ is arbitrarily larger than $\sigma^{(2)}_{[1:T]} + \Sigma^{(2)}_{[1:T]}$. 
 \end{lemma}
   \begin{proof}
 Since the distribution mean minimizes the least squares error, $\Exp{\Var_T} \geq \sigma^{(2)}_{[1:T]}$  always holds. Using the same argument, we have
 \begin{align*}
 \Sigma^{(2)}_{[1:T]} = \sum_{t=1}^T   \norm{\nFf{t}- \nFf{t-1}}^2 \leq  4 \sum_{t=1}^T  \norm{ \nabla f_t- \nFf{t}}^2 \leq 4 \sum_{t=1}^T   \norm{ \nabla f_t-  \mu^*_T}^2 = 4 \Var_T.
 \end{align*}
 Consider $\cX = [-1,1]$ and $f(x,\xi) = \xi x$. Now suppose the SEA chooses truncated normal distribution with mean $-2$ for the first $T/2$ rounds, then truncated normal distribution with mean $2$ for the remaining rounds. Assume in both cases the variance is $\sigma^2$ and truncation is in range $[-G,G]$ for $G>0$.
   Hence, for sufficiently large $T$, $\mu_T = 0$ and $\Var_T = \sum_{t=1}^T \norm{f_t}^2 \propto T G^2$.  However, $\sigma_{[1:T]}^{(2)}$ is equal to $T\sigma^2$ which can be considerably smaller than $TG^2$. The price for the distribution switch is captured with a small constant overhead by $\Sigma_{[1:T]}^{(2)}\leq2G^2$. Thus, $\sigma_{[1:T]}^{(2)} + \Sigma_{[1:T]}^{(2)} \propto T\sigma^2 + 2G^2$.
  \end{proof}

\begin{lemma}
\label{lem:Dp:comp}
Define $D_p = \sum_{t=1}^T \sup_{x\in\cX}\tnorm{\nabla f_t(x) - \nabla f_{t-1}(x)}^2_p$ . In the SEA framework,
\begin{align*}
\Exp{D_2} \geq \frac{1}{2}\left(\sigma^{(2)}_{[1:T]} + \Sigma^{(2)}_{[1:T]} \right).
\end{align*} 
Furthermore, there exist instances such that $\Exp{D_2} \gg \sigma_{[1:T]}^{(2)} + \Sigma_{[1:T]}^{(2)}$.
\end{lemma}
\begin{proof}
We shall in fact prove that $\E[D_2] \geq \max( \sigma_{[1:T]}^{(2)} , \Sigma_{[1:T]}^{(2)}) $, which directly implies the first part of the statement. Fix a time $t \geq 2$ and let $\mathcal G_{t-1}$ denote the $\sigma$-algebra generated by $(x_1, \xi_1, \dots, \xi_{t-1}, x_t)$. Then for any $x \in \cX$, the variable $F^{t}(x) = \E[f_t(x) \vert \mathcal G_{t-1}] \vphantom{}$ is $\mathcal G_{t-1}$-measurable, and we have
\begin{multline*}
  \E\bigg[\sup_{x \in \cX} \|\nabla f_t(x) - \nabla f_{t-1}(x)\|^2 \, \Big \vert\,  \mathcal G_{t-1}\bigg]
  \geq \sup_{x \in \cX} \E\Big[ \|\nabla f_t(x) - \nabla f_{t-1}(x)\|^2 \mid \mathcal G_{t-1}\Big] \\
  \geq \sup_{x \in \cX} \E\Big[ \|\nabla f_t(x) - \nabla F^t(x)\|^2 \mid \mathcal G_{t-1}\Big] = \sigma_t^2 \, 
\end{multline*}
since $f_{t-1} = f( \, \cdot \, , \xi_{t-1})$ is $\mathcal G_{t-1}$-measurable, and $\nabla F^t(x)  = \E \big[ \nabla f^t(x) \, \vert \, \mathcal G_{t-1} \big]$. Therefore, by conditioning on $\mathcal G_{t-1}$ at time step $t$, and applying the tower rule, we obtain
\begin{equation*}
  \E \Bigg[\sum_{t=1}^T \sup_{x \in \cX} \|\nabla f_t(x) - \nabla f_{t-1}(x)\|^2  \Bigg]
  \geq \E \bigg[\sum_{t=1}^T \sigma_t^2 \bigg]= \sigma_{[1:T]}^{(2)} \, . 
\end{equation*}
The lower bound by $\Sigma_{[1:T]}^{(2)}$ holds by a direct application of Jensen's inequality, and by swapping suprema and expectations.

For the second part of the lemma, consider the $d$-dimensional euclidean ball $\cX = B_d(1) \subset \RR^d$, and $\Xi = [d]$. Define 
\begin{align*} 
 f(x,i) = x_i^2 / 2 \, . 
\end{align*}
Consider a fully stochastic (i.i.d.) SEA picking $\xi_t \in [d]$ uniformly at random at every time step. Then $\Sigma_{[1:T]}^{(2)} = 0$. We shall now see that $\sigma_{[1: T]}^{(2)} \leq T/d$. Indeed, for any $x \in \cX$ and $t\in [T]$, then $F^t$ does not depend on $t$ and its value is 
\begin{equation*}
  F(x) = \E_{I\sim \cD} [f(x, I)] = \frac{1}{2d} \sum_{i = 1}^d x_i^2 \, ,
\end{equation*}
which is a convex and smooth function.  We can upper bound the variance, as for any $x \in \cX$, 
\begin{equation*}
  \E_{I \sim \mathcal D}  \big[\|\nabla f(x, I) -  \nabla F(x) \|^2 \big]
  \leq \E_{I \sim \mathcal D} \big[\| \nabla f(x, I)  \|^2 \big]
  = \E_{I \sim \mathcal D} \big[\| x_I e_I  \|^2\big]
  = \frac{1}{d} \sum_{i = 1}^d x_i^2 \leq \frac{1}{d} \, . 
\end{equation*}
Therefore, after the taking the supremum over $x \in \cX$, we see that $\sigma_t^2 \leq 1 / d$. On the other hand, for any $I, J \in [d]$, we have 
\begin{equation*}
  \|\nabla f(x, I) -  \nabla f(x, J) \|^2
   = \| x_I e_I - x_J e_J\|^2 = (x_I^2 + x_J^2)\mathbf 1 \{ I \neq J \} \, . 
\end{equation*}
The maximum in the ball of this difference is reached at $x = \sqrt{2}/2 (e_I + e_J)$ and 
\begin{equation*}
  \max_{x \in \cX} \|\nabla f(x, I) -  \nabla f(x, J) \|^2 = \mathbf 1 \{ I \neq J \} \, . 
\end{equation*}
Therefore, if $I$ and $J$ are independent and uniformly distributed over $[d]$, then
\begin{equation*}
  \E_{(I,J) \sim \cD \otimes \cD} \Big[\max_{x \in \cX} \|\nabla f(x, I) -  \nabla f(x, J) \|^2 \Big]= \mathbb P_{(I,J) \sim \cD \otimes \cD}[I \neq J] = \Big(1 - \frac{1}{d} \Big)^2 \geq 1 / 4 \, . 
\end{equation*}

Summarizing the above inequalities, we have built an example in which
\begin{equation*}
  \E [D_2 ]
  \geq \frac{T}{4}  
  \gg \frac{T}{d} 
  \geq \sigma_{[1:T]}^{(2)}+  \Sigma_{[1:T]}^{(2)} \, . 
\end{equation*}
In particular, the expectation of the variation $D_2$ can be arbitrarily larger than the cumulative variance, and our bounds are then tighter than those obtained via a direct application of known results. 
\end{proof}

\section{Proofs of Section \ref{sec:mainResults}}
   
  \subsection{Proof of Theorem \ref{thm:VarStepsize}}
  \label{Appendix:VarStepsize}
  To prove Theorem~\ref{thm:VarStepsize}, we need the following well-known result from the literature.
 \begin{lemma}
 \label{Regret:Lsmooth:lemma}
 Suppose $f_t(\, \cdot\,)$ are convex for all $t \in [T]$ and $\psi_t(\, \cdot \,) = \frac{2}{\eta_t}\norm{\, \cdot\,}^2$. Further, let $g_t \in \partial f_t(x_t)$ and assume  $\eta_t \geq \eta_{t+1}$.  Then the regret for OFTRL is bounded by 
 \begin{equation}
   R_T \leq  \frac{D^2}{\eta_T} + 
   \sum_{t=1}^T \Big( \langle g_t - M_t , \, x_{t} - x_{t+1} \rangle - \frac{1}{\eta_t} \| x_{t+1} - x_t \|^2\Big)  .
 \end{equation}
 \end{lemma}
 \begin{proof}
 Denote $F_t (x) \eqdef \psi_t (x) + \sum_{s=1}^{t-1} f_s(x)$. Note that $F_t$ is $\frac{2}{\eta_t}$-strongly convex. Thus, Thm 7.29 in \cite{Orabona2019AMI} gives 
 \begin{align*}
 R_T(u) 
 &\leq \psi_{T+1}(u) - \psi_1(x_1)\\
 &\qquad + \sum_{t=1}^T \Big( \dprod{g_t - M_t, x_t - x_{t+1}} - \frac{1}{\eta_t}\norm{x_t - x_{t+1}}^2 + \psi_{t}(x_{t+1}) - \psi_{t+1}(x_{t+1}) \Big).\\
 \intertext{Since $\psi_t(x) - \psi_{t+1}(x) \leq 0$ and $\psi_{T+1}(u) \leq \frac{D^2}{\eta_T}$, this gives}
 &\leq \frac{D^2}{\eta_T} + 
   \sum_{t=1}^T \Big( \langle g_t - M_t , \, x_{t} - x_{t+1} \rangle - \frac{1}{\eta_t} \| x_{t+1} - x_t \|^2\Big) \, .\qedhere
 \end{align*}
 \end{proof}

 \thmVarStepsize*
  \begin{proof}[Proof of Theorem \ref{thm:VarStepsize}.]
    Write $g_t = \nabla f(x_t , \xi_t)$ and denote by $\E$ the expectation with respect to all the randomness. Using the Optimistic FTRL bound from Lemma \ref{Regret:Lsmooth:lemma}, 
    \begin{align*}
      \sum_{t= 1}^T \langle g_t, x_t - u \rangle
      &\leq \frac{ D^2}{\eta_T} + \sum_{t= 1}^T \Big(\langle g_t - M_t, \, x_t - x_{t+1} \rangle - \frac{1}{2\eta_t} \|x_t - x_{t+1}\|^2 \Big)  \\
      &\leq \frac{D^2}{\eta_T} + \sum_{t= 1}^T \frac{\eta_t}{2}\|g_t - M_t\|^2  - \sum_{t=1}^T\frac{1}{2\eta_t} \|x_t - x_{t+1}\|^2,
    \end{align*}
    where the last inequality is obtained by separating the negative norm term in two parts, and keep half of it in the regret bound.
    Let us plug in the value $\eta_t$, 
    \begin{equation*}
      \eta_t  = D^2 \bigg( \nu + \sum_{s=1}^{t-1} \eta_s \|g_s - M_s\|^2 \bigg)^{-1} \,,
    \end{equation*}
    and use the fact that $\eta_t \leq D^2 / C$ to further upper bound the deterministic regret by
    \begin{equation}\label{eq:sum_with_etas}
       \nu + \frac{3}{2}\sum_{t =1 }^T \eta_t \|g_t - M_t\|^2 - \frac{\nu}{2D^2}\sum_{t=1}^T \|x_t - x_{t+1}\|^2. 
    \end{equation}
  To bound the second term above, we first compute
    \begin{align*}
      \Bigg(\sum_{t=1}^T \eta_t \|g_t - M_t\|^2 \Bigg)^2 &= \sum_{t=1}^T  \sum_{s=1}^{T}  \eta_s \|g_s - M_s\|^2  \eta_t \|g_t - M_t\|^2\\
      &= 2\sum_{t=1}^T  \left(\sum_{s=1}^{t-1}  \eta_s \|g_s - M_s\|^2 \right) \eta_t\|g_t - M_t\|^2 + \sum_{t= 1}^T \eta_t^2 \|g_t - M_t\|^4 \,. \\
      \intertext{Since $\eta_t \leq D^2/ (\sum_{s=1}^{t-1}  \eta_s \|g_s - M_s\|^2)$}
      &\leq 2 D^2\sum_{t=1}^T \|g_t - M_t\|^2 + \left(\frac{\sum_{t= 1}^T\eta_t^2 \|g_t - M_t\|^4}{\sum_{t= 1}^T\eta_t \|g_t - M_t\|^2} \right) \sum_{t=1}^T \eta_t \|g_t - M_t\|^2 \, . 
    \end{align*}
    Now we use  the fact $X^2 \leq 2A + BX$ implies $X \leq \sqrt{2A} + B$ for $A, B >0$. Hence
    \begin{align*}
      \sum_{t=1}^T \eta_t \|g_t- M_t\|^2 
      &\leq D\sqrt{2 \sum_{t=1}^T \|g_t - M_t\|^2} + \frac{\sum_{t= 1}^T\eta_t^2 \|g_t - M_t\|^4}{\sum_{t= 1}^T\eta_t \|g_t - M_t\|^2}. \\
      \intertext{Next we use $\eta_t \leq D^2 / \nu$ and $\|g_t - M_t\|^2 \leq 4G^2$ to bound the last term. The sum satisfies $\sum_{t= 1}^T\eta_t^2 \|g_t - M_t\|^4 \leq (4G^2D^2/\nu)\sum_{t= 1}^T \eta_t \|g_t - M_t\|^2$ and we can further bound the term above } 
      &\leq D\sqrt{2 \sum_{t=1}^T \|g_t - M_t\|^2}  + \frac{4D^2G^2}{\nu} \, . 
    \end{align*} 
    All in all, we have
    \begin{equation}\label{eq:ftrl-reg-bound-pre-expectation}
      \sum_{t=1}^T \langle g_t, x_t - u\rangle
      \leq \frac{3\sqrt 2 D}{2}\sqrt{\sum_{t=1}^T \|g_t -M_t\|^2 }  + \nu + \frac{4D^2G^2}{\nu} - \frac{\nu}{2D^2} \sum_{t=1}^T \|x_{t+1} -x_t\|^2 \, . 
    \end{equation}

    Note that we have not used any assumption on the expected $F^t$'s, and in particular not the smoothness. Therefore, even if the expected losses are not smooth, our analysis already entails that if $\|M_t\|\leq G$
    \begin{equation}\label{eq:worst-case-dgsqrtt}
      \sum_{t=1}^T \langle g_t, x_t - u\rangle \leq 
      \frac{3\sqrt 2 D G}{2}\sqrt{4T} + \nu + \frac{4D^2G^2}{\nu} \, , 
    \end{equation}
    proving the final claim of the statement.

    Let us now proceed with the proof of the finer results. We use the value of $M_t$, together with the fact that (by convexity of $a \mapsto \|a\|^2$), for any $t \geq 2$,
    \begin{align*}
        \norm{g_t- g_{t-1}}^2 &\leq  4\norm{ g_t - \nFf{t}(x_{t})}^2 + 4 \norm{ \nFf{t}(x_{t}) - \nFf{t}(x_{t-1})}^2 \nonumber\\
        \qquad &   +  4\norm{ \nFf{t}(x_{t-1}) -  \nFf{{t-1}}(x_{t-1})}^2 + 4\norm{ \nFf{{t-1}}(x_{t-1}) - g_{t-1}}^2  \nonumber \\
        &\leq 4\norm{ g_t - \nFf{t}(x_{t})}^2 + 4L^2 \norm{ x_{t} - x_{t-1}}^2 \nonumber \\
        \qquad &   +  4\norm{ \nFf{t}(x_{t-1}) -  \nFf{{t-1}}(x_{t-1})}^2 + 4\norm{ \nFf{{t-1}}(x_{t-1}) - g_{t-1}}^2 \, . 
    \end{align*} 
    Therefore, using the inequality $\sqrt{a + b} \leq \sqrt{a} + \sqrt{b}$, as well as $\|g_1\| \leq G$ and reorganizing the terms,
    \begin{multline*}
      \sqrt{\sum_{t=1}^T \|g_t - g_{t-1}\|^2}
      \leq G +  \sqrt{  8 \sum_{t=2}^T \|g_t - \nabla F^t(x_t)\|^2} \\
      + 2L \sqrt{\sum_{t=2}^T \|x_t - x_{t-1}\|^2}
      + 2 \sqrt{\sum_{t=2}^T \|\nFf{t}(x_{t-1}) -  \nFf{{t-1}}(x_{t-1})\|^2} \, .
    \end{multline*}
    The sum of the variations of $x_t$'s can be cancelled thanks to the negative term in \eqref{eq:ftrl-reg-bound-pre-expectation}, as
    \begin{equation*}
      \frac{3\sqrt 2 D}{2} 2L\sqrt{\sum_{t=1}^{T-1} \|x_{t+1} - x_t\|^2} - \frac{\nu}{2D^2} \sum_{t=1}^T \|x_{t+1} - x_t\|^2 
      \leq \sup_{X\geq0}\bigg\{ 3\sqrt{2} LD X  - \frac{\nu}{2D^2}X^2 \bigg\} = \frac{9L^2D^4}{\nu} \, .
    \end{equation*}
    After replacing these bounds in \eqref{eq:ftrl-reg-bound-pre-expectation}, we have obtained the regret bound
    \begin{multline}\label{eq:bound_before_expectation}
      \sum_{t=1}^T \langle g_t, x_t - u \rangle
      \leq 6 D \sqrt{\sum_{t=1}^T \|g_t - \nabla F^t(x_t)\|^2}  + 3\sqrt{2}D \sqrt{\sum_{t=1}^{T-1} \| \nabla F^{t+1}(x_t)-  \nabla F^t(x_t)\|^2} \\
      + \frac{3\sqrt 2 DG}{2} + \nu + \frac{4 D^2G^2}{\nu} + \frac{9 L^2D^4}{\nu} \, .
    \end{multline}
    We will then take expectations in the inequality above. To bound the right-hand side, let us denote by $\mathcal G_t = \sigma (x_1, \xi_1, \dots, x_{t-1}, \xi_{t-1}, x_t)$, then $\xi_t$ is distributed according to $\mathcal D_t$ given $\mathcal G_t$, and since $x_t$ is $\mathcal G_t$-measurable, therefore
    \begin{multline*}
      \E \big[ \| \nabla f(x_t, \xi_t) - \nabla F^t(x_t)\|^2  \mid \mathcal G_t \big]
      \leq \E_{\xi \sim \mathcal D_t} \big[ \| \nabla f(x_t, \xi) - \nabla F^t(x_t)\|^2 \big] \\
      \leq \sup_{x \in \cX} \E_{\xi \sim \mathcal D_t} \big[ \| \nabla f(x, \xi) - \nabla F^t(x)\|^2 \big] = \sigma_t^2 \, . 
    \end{multline*}
    Therefore, by the tower rule, 
    \begin{equation}\label{eq:sigma_xt}
      \E \big[ \| \nabla f(x_t, \xi_t) - \nabla F^t(x_t)\|^2\big]
      = \E \Big[ \E \big[ \| \nabla f(x_t, \xi_t) - \nabla F^t(x_t)\|^2  \mid \mathcal G_t \big] \Big]
      \leq \E\big[ \sigma_t^2 \big] \, .
    \end{equation}
    The final result follows from taking expectations in \eqref{eq:bound_before_expectation}, applying Jensen's inequality, incorporating~\eqref{eq:sigma_xt} and using the definitions of $\bar \sigma_T$ and $\bar \Sigma_T$.    
  \end{proof}


\subsection{Proof of Theorem \ref{theorem:regretLsmoothLB}}
 \label{AppendixLBsmooth}
 \LsmoothLB*
 \begin{proof}[Proof of Theorem \ref{theorem:regretLsmoothLB}.]
 Suppose we are given two parameters $\hat \sigma_T $ and $\hat \Sigma_T $, we show that there exists a sequence of distributions $\cD_1, \dots \cD_T$ such that the expected regret is at least $\Omega(D (\hat \sigma_T  + \hat \Sigma_T )\sqrt{T})$.
 Let $1 \leq a < b $ be constants such that $a \geq \frac{1}{2} b$. 
 Since for any closed convex set   there exist an affine transformation which mapps it to the interval $[a,b]$, we assume without loss of generality that $\cX = [a,b]$.

 Suppose     $f: \cX \times \Xi \rightarrow \RR$ and let $z^{(\sigma)},z^{(\Sigma)}\in\RR$. Assume the gradients have the form
 \begin{align*}
 \nabla f(x, \xi) = z^{(\sigma)} \qquad \text{ or } \qquad \nabla f(x, \xi) =  z^{(\Sigma)}.
 \end{align*}
 Assume SEA chooses each case with probability $1/2$. The idea is to construct two sequences $\{z^{(\sigma)}_{t} \}_{t \in [T]} $  and $\{z^{(\Sigma)}_{t} \}_{t \in [T]} $ such that these sequences have  at least  $\Omega(D \hat \sigma_T \sqrt{T})$ and  $\Omega(D \hat \Sigma_T \sqrt{T})$ expected regret, respectively. 
 Therefore, let $ x_t$ denote the learners choice in round $t$ and 
 define linearised regret with respect to  $\{z^{(\sigma)}_{t} \}_{t \in [T]} $  and $\{z^{(\Sigma)}_{t} \}_{t \in [T]} $. 
 \begin{align*}
 R_T^{\sigma} = \min_{u \in \cX}\sum_{t\in [T]} \dprod{z^{(\sigma)}_{t}, x_t - u}\quad  \text{ and } \quad R_T^{\Sigma} = \min_{u \in \cX} \sum_{t\in  [T]} \dprod{z^{(\Sigma)}_{t}, x_t  - u }.
 \end{align*}
 
 \paragraph{Case $R_T^{\Sigma}$:} Let $G = \hat \Sigma_T$. Define $\cz : \cX \rightarrow \RR$, $\cz(x) = \frac{1}{4b} G x^2$. Then $\cz$ is $G$-Lipschitz, smooth and $\cz'(x) \in [\frac{1}{2} G , G]$ for any $x \in \cX$. 
 Let $\{\epsilon_t\}_{t \in  [T]}$ be an i.i.d. sequence of Rademacher random variables, that is, $\Prob{\epsilon_t = 1} = \Prob{\epsilon_t = -1} = 1/2$. 
 The sequence $\{z^{(\Sigma)}_t\}_{t \in [T]}$ is defined as
 \begin{align*}
  z^{(\Sigma)}_t = \begin{cases}
  0 &\text{ if $t$ even}  \\
  \epsilon_t \cz'(x_t) &\text{ if $t$ odd} .
  \end{cases}
 \end{align*}
 Using that $\Exp{\epsilon} = 0$ together with the definition of $z^{(\Sigma)}_t$ gives 
 \begin{align*}
 \ExpD{\epsilon \sim \Rad}{R_T^{\Sigma}} &= \ExpD{\epsilon \sim \Rad}{ \min_{u \in \cX} \sum_{t=1}^T \dprod{z^{(\Sigma)}_{t}, x_t  - u }} \\
 &= \ExpD{\epsilon \sim \Rad}{ \max_{u \in \cX}  \sum_{\substack{t=1\\t \text{ odd}}}^T \dprod{  \epsilon_t \cz'(x_t) ,    u }} \\
 &\geq \frac{G}{4} \ExpD{\epsilon \sim \Rad}{ \max_{u \in \cX} \sum_{t=1 }^{\frac{T}{2}}   \epsilon_t    u } .
 \intertext{Now use that for a linear function $l(x)$, $\max_{x \in [a,b]} l(x) = \max_{x \in \{a,b\}}l(x) = \frac{l(a+b)}{2} + \frac{\absv{l(a-b)}}{2}$. }
 &= \frac{G}{4} \ExpD{\epsilon \sim \Rad}{ \max_{u \in\{a,b\}}  \sum_{t=1}^{\frac{T}{2}}   \epsilon_t   u } \\
 &= \frac{G}{8} \ExpD{\epsilon \sim \Rad}{  \sum_{t=1}^{T / 2}   \epsilon_t(    a + b )} +    \frac{G}{8} \ExpD{\epsilon \sim \Rad}{ \absv{\sum_{t =1}^{T/ 2} G  \epsilon_t(    a - b ) }} \\
 &= \frac{G}{16} \ExpD{\epsilon \sim \Rad}{      \absv{ \sum_{t=1}^T   \epsilon_t  (  a - b)  }}. \\
 \intertext{Where we have used $\Exp{\epsilon} = 0$ again. Now we use that by definition $D = \sup_{x,y \in \cX}\norm{x-y}$.}
 &= \frac{GD}{16}  \ExpD{\epsilon \sim \Rad}{      \absv{ \sum_{t=1}^T  \epsilon_t    }} \geq \frac{1}{32}  D  \sqrt{G^2 T}.
 \end{align*}
 In the last step we have used the Khintchine inequality. 
  Now note that $G^2 \geq \frac{1}{2}\sup_{x \in \cX} \norm{\cz'(x)}^2 = \frac{1}{2}\sup_{x \in \cX} \norm{\cz'(x) - 0}^2$. Due to the definition of the sequence $\{z_t^{(\Sigma)}\}_{t \in [T]}$,  if $\norm{\nabla f(x,\xi^{t})} \neq 0$, then $\norm{\nabla f(x,\xi^{t-1})} = 0$. Thus, 
  $\frac{1}{2}\sup_{x \in \cX} \norm{\cz'(x) - 0}^2 = \frac{1}{2}\sup_{x \in \cX} \norm{\nabla f(x,\xi^{t}) - \nabla f(x,\xi^{t-1})}^2\,$ for any $t \in [T]$. 
  
 \noindent Thus, $\sqrt{G^2 T } = \sqrt{T /(2T) \sum_{t=1}^T \sup_{x \in \cX} \norm{\nabla f(x,\xi^{t}) - \nabla f(x,\xi^{t-1})}^2} =  \bar \Sigma_T \sqrt{T/2}$. Setting the value $G = \hat \Sigma_T$ completes this part of the proof.

 \paragraph{Case $R_T^{\sigma}$: } We will show this part by contradiction. Suppose that $\cD$ is a distribution such that the variance of the gradients $\sigma$ is equal to $\hat \sigma_T$. Suppose the SEA picks this distribution every round and assume for contradiction that
   $\Exp{R_T^\sigma} \leq o(D  \sigma \sqrt{T})$.  Using online-to-batch conversion gives a convergence bound of order $o(D \sigma / \sqrt{T})$ which contradicts well-known lower bounds from stochastic optimization (c.f., \cite{6142067, doi:10.1137/1027074} Section~5).
 \end{proof}

  
 \subsection{Proof of Theorem \ref{theorem:scRegretbound}}

  \label{appendix:scRegretbound}
  We first need a well known result for OFTL for strongly convex loss functions.
   
  \begin{lemma}
  \label{lem:scRegret}
  Suppose $f_t(\, \cdot\,)$ are $\mu$-strongly convex for all $t \in [T]$ and $\psi_t(\, \cdot \,) =  0$. Further, let $m_t:\cX \rightarrow \RR$ denote 
  the optimistic prediction, and $g_t \in \partial f_t(x_t)$, $M_t \in \partial m_t(x_t)$.  Then the regret for OFTRL is bounded by 
  \begin{align*}
  R_T(u) \leq   \sum_{t=1}^T  \Big(  \dprod{g_t -M_t,x_t - x_{t+1}} - \frac{t \mu}{2} \norm{x_t - x_{t+1}}^2 \Big).
 \end{align*}
  \end{lemma}
  This is a well known result and can be found in the literature, e.g., \cite{Orabona2019AMI}. We include a short proof for completeness. 
  \begin{proof}
  Let $\bar F_t(x) = \sum_{s=1}^{t-1} f_t(x)$ and $G_t \in \partial \bar F_{t+1}(x_t)$. Note that $\bar F_t$ is $[(t-1)\mu]$-strongly convex. From standard analysis (see, e.g., \cite{Orabona2019AMI} Lem. 7.1) we obtain
  \begin{align*}
  \sum_{t=1}^T [f_t(x_t) - f_t(u)] &= \underbrace{\bar F_{T+1}(x_{T+1}) - \bar F_{T+1}(u)}_{\leq 0} +   \sum_{t=1}^T[\underbrace{ \bar F_t(x_t) +   f_t(x_t)}_{=\bar F_{t+1}(x_t)} - \bar F_{t+1}(x_{t+1})].\\
  &\leq \sum_{t=1}^{T} [\bar F_{t+1} (x_t) - \bar F_{t+1}(x_{t+1})] \\
  &\leq \sum_{t=1}^T \Big( \dprod{G_{t}, x_t - x_{t+1}} - \frac{t\mu}{2} \norm{x_t - x_{t+1}} \Big).
  \end{align*}
  Due to convexity, $G_{t} \in \partial \bar F_{t+1}(x_t) = \partial \bar F_{t}(x_t) \cap \partial f_t(x_t) $ and due to update operation $ 0 \in \partial \bar F_t(x_t) \cap \partial m_t(x_t)$. Thus, there exist $g_t \in \partial f_t(x_t)$ and $M_t \in \partial m_t(x_t)$ such that $G_{t} = g_t - M_t$, which completes the proof. 
  \end{proof}
  
  \scRegretbound*
\begin{proof}[Proof of Theorem \ref{theorem:scRegretbound}.]
  Thanks to the strong convexity assumption \ref{A4},
 \begin{equation*}
   \fF{t}(x_t) - \fF{t}(x) \leq \langle x_t - x, \nFf{t}(x_t) \rangle  - \frac{\mu}{2} \|x - x_t\|^2 \, .
 \end{equation*}
 Taking expectation and using the definition of $\ell_t$ gives 
 \begin{align*}
   \Exp{\fF{t}(x_t) - \fF{t}(x)} &\leq \Exp{\dprod{ x_t - x, \nFf{t}(x_t) }  - \frac{\mu}{2} \|x - x_t\|^2}\\
   & = \Exp{\dprod{ x_t - x, \nabla f(x_t, \xi_t) } - \frac{\mu}{2} \|x - x_t\|^2} \\
   & = \Exp{\dprod{ x_t - x, g_t }  - \frac{\mu}{2} \|x - x_t\|^2 }  \\
   & = \Exp{\ell_t(x_t) - \ell_t(x)} \, .
 \end{align*}
 Now each function $\ell_t$ is $\mu$-strongly convex, and $\nabla \ell_t(x_t) = g_t$.  Thus we can apply   Lemma \ref{lem:scRegret}
 \begin{align*}
   \sum_{t=1}^T \ell_t(x_t) - \ell_t(x)
   &\leq \sum_{t=1}^T \Big( \dprod{g_t - M_t, x_t - x_{t+1}} - \frac{\mu t}{2}\norm{x_t - x_{t+1}}^2 \Big). \\
   & \leq \sum_{t=1}^T \left( \frac{1}{\mu t} \|g_t - M_t\|^2  +  \left(\frac{\mu t}{4}  - \frac{\mu t}{2} \right)\norm{x_t - x_{t+1}}^2 \right) .
 \end{align*}
 where we used the inequality $\dprod{a,b} \leq \frac{1}{2c}\norm{a}^2 + \frac{c}{2}\norm{b}^2$.
 Once again, keeping the negative norm term is crucial. Indeed, using the convexity of $x \mapsto \|x\|^2$ and the smoothness assumption on $\nabla F^t$, we get that for $t \geq 2$
\begin{align*}
  \norm{g_t- g_{t-1}}^2 &\leq  4\norm{ g_t - \nFf{t}(x_{t})}^2 + 4 \norm{ \nFf{t}(x_{t}) - \nFf{t}(x_{t-1})}^2 \nonumber\\
  \qquad &   +  4\norm{ \nFf{t}(x_{t-1}) -  \nFf{{t-1}}(x_{t-1})}^2 + 4\norm{ \nFf{{t-1}}(x_{t-1}) - g_{t-1}}^2  \nonumber \\
  &\leq 4\norm{ g_t - \nFf{t}(x_{t})}^2 + 4L^2 \norm{ x_{t} - x_{t-1}}^2 \nonumber \\
  \qquad &   +  4\norm{ \nFf{t}(x_{t-1}) -  \nFf{{t-1}}(x_{t-1})}^2 + 4\norm{ \nFf{{t-1}}(x_{t-1}) - g_{t-1}}^2 \, . 
\end{align*}
So that, upper bounding the first term $\ell_1(x_1) - \ell_1(u) \leq GD$ we get
\begin{multline*}
    \sum_{t=1}^T \ell_t(x_t) - \ell_t(u) \\
    \leq \sum_{t=2}^T \frac{1}{\mu t} \biggl( 4 \|g_t  - \nabla F^t(x_t)\|^2 
                          + 4 \|g_{t-1} -  \nabla F^{t-1}(x_{t-1})\|^2
                          + 4 \| \nabla F^{t}(x_{t-1}) - \nabla F^{t-1}(x_t)\|^2
                        \biggr) \\
        + \sum_{t= 1}^T \biggl( \frac{4 L^2}{\mu (t+1)} - \frac{\mu t}{4} \biggr) \|x_t - x_{t+1}\|^2  + GD\, .
\end{multline*}
The indices can be simplified by noting that,
\begin{equation*}
  \sum_{t=2}^T \frac{4}{\mu t} \|g_{t-1} - \nabla F^{t-1}(x_{t-1})\|^2
  \leq \sum_{t=2}^T \frac{4}{\mu (t-1)}  \|g_{t-1} - \nabla F^{t-1}(x_{t-1})\|^2
  \leq \sum_{t=1}^T \frac{4}{\mu t}  \|g_t - \nabla F^t(x_t)\|^2 \, . 
\end{equation*}
To recover
\begin{multline}
\label{sc:intermed:result}
  \sum_{t=1}^T \ell_t(x_t) - \ell_t(u) 
  \leq \sum_{t=1}^T \frac{8}{\mu t} \|g_t  - \nabla F^t(x_t)\|^2 
      + \sum_{t=2}^T \frac{4}{\mu t} \| \nabla F^{t}(x_{t-1}) - \nabla F^{t-1}(x_t)\|^2  \\
      + \sum_{t= 1}^T \biggl( \frac{4 L^2}{\mu t } - \frac{\mu t}{4} \biggr) \|x_t - x_{t+1}\|^2  + GD\, .
\end{multline}
Define the condition number $\kappa = L / \mu$. Then, for $t \geq 16 \kappa$, we have
$
   \frac{4L^2}{\mu t} - \frac{\mu t}{4} \leq 0 \, . 
$
Therefore the second term can be bounded independently of $T$
\begin{equation*}
  \sum_{t=1}^{\lceil  16 \kappa \rceil} 
    \biggl( \frac{4L^2}{\mu t} - \frac{\mu t}{4} \biggr) D^2
  \leq \frac{4L^2 D^2}{\mu} \sum_{t=1}^{\lceil 16 \kappa \rceil} \frac{1}{t} 
  \leq \frac{4L^2 D^2}{\mu} \log (1 + 16 \kappa)                                                
  \, .
\end{equation*}
 Combining all bounds, and incorporating the definition of $\sigma_{\max}$ and $\Sigma_{\max}$, 
 \begin{equation*}
   \Exp{R_T(u)} \leq  
    \frac{1}{\mu} \left( 8 \sigma_{\max}^2
    + 4  \Sigma_{\max}^2 \right)\log T  + 4 D^2 L \kappa\log( 1+  16\kappa) + GD\, .\qedhere
 \end{equation*}
 \end{proof}


  \subsection{Proof of Theorem \ref{sc:lowerbound}}
  \label{app:sc:lowerbound}
   \sclowerbound*
  \begin{proof}[Proof of Theorem  \ref{sc:lowerbound}.]

  Let $\hat\sigma_{\max}, \hat\Sigma_{\max}$ be given parameters and set $G = \max(\hat\sigma_{\max},\hat\Sigma_{\max} / 2)$.
  We want to show that there exist sequence of distributions $\cD_1, \dots \cD_T$ such that 
  \begin{enumerate}
  \item $\sigma_{\max} = \hat \sigma_{\max}$ and $\Sigma_{\max} = \hat \Sigma_{\max}$.
  \item $\Exp{R_T(u)} \geq c \frac{1}{\mu}(\Sigma^2_{\max} + \sigma^2_{\max}) \log T$ for some constant $c > 0$. 
  \item $F_1, \dots, F_T$ are $\mu$-strongly convex.  
  \end{enumerate}
  Consider the iterations up to $T-3$.
   From Corollary 20 in \cite{pmlr-v19-hazan11a} we obtain an $ \Omega(\frac{1}{\mu} G^2 \log (T-3))$ lower bound on the expected regret. Thus, there exist a realization $\xi_1, \dots, \xi_{T-3}$ and corresponding $\mu$-strongly convex functions $f(\,\cdot\,,\xi_1),\dots, f(\,\cdot\,,\xi_{T-3})$ such that with respect to this realization, $R_T(u)\geq\Omega(\frac{1}{\mu} G^2 \log (T-3))$. We now let   $\delta_1, \dots \delta_{T-3}$ be the Dirac measure corresponding to this realization. 
   Then, $\sigma_{\max} = 0$ and  
   $  \max_{t \in [T-3]} \sup_{x\in\cX}\tnorm{\nFf{t}(x) - \nFf{t-1}(x)} \leq 2G$.   But we do not necessarily have that $\sigma_{\max} = \hat \sigma_{\max}$ and $  \max_{t \in [T-3]} \sup_{x\in\cX}\tnorm{\nFf{t}(x) - \nFf{t-1}(x)} = \hat \Sigma_{\max}$.
   To guarantee this  we want to choose $\cD_{T-2},\cD_{T-1},\cD_T$, such that
   \begin{enumerate}
   \item \label{cond1}$  \sup_{x\in\cX}\tnorm{\nFf{T-2}(x) - \nFf{T-1}(x)}  = \hat \Sigma_{\max}$ and  $\tnorm{\nFf{T-i}(x)} \leq G$ for $i = 1,2$.
   \item  \label{cond2}$\sigma_{T} = \hat\sigma_{\max}$ and $\tnorm{\nFf{T}(x)} \leq G$. 
   \end{enumerate} 
 To satisfy Condition \ref{cond1},  let $\cD_{T-2},\cD_{T-1}$ be Dirac measures, such that $ \nFf{T-2}(x) =  -\nFf{T-1}(x)$ and $\tnorm{\nFf{T-2}(x)} = \frac{1}{2}\hat\Sigma_{\max}  $. Then, by definition of $G$, we know that  $\tnorm{\nFf{T-2}(x)}  \leq G $ and $  \sup_{x\in\cX}\tnorm{\nFf{T-2}(x) - \nFf{T-1}(x)}  = \hat \Sigma_{\max}$. Condition \ref{cond2} can be satisfied by setting $\cD_T$ to be any distribution with sufficient variance. This gives 
 \begin{equation}
 \label{bound1inter}
 \Exp{R_T(u)} \geq c\frac{1}{\mu} \left(\Sigma^2_{\max} + \sigma^2_{\max} \right)\log(T-3)  - \Exp{\sum_{t=T-2}^T f(x_t,\xi_t ) - f(u,\xi_t)}.
 \end{equation}
 Now it remains to show that the last term is negligible. Indeed, from the upper bound, we know
 \begin{align*}
 \Exp{\sum_{t=T-2}^T [f(x_t,\xi_t ) - f(u,\xi_t)]} \leq   \frac{3}{(T-2)\mu}\left(\Sigma^2_{\max} + \sigma^2_{\max} \right).
 \end{align*}
 Hence, for any $T\geq 10$, we get $  \tfrac{3}{(T-2)\mu}\left(\Sigma^2_{\max} + \sigma^2_{\max} \right) \leq  \frac{1}{2\mu}  \left(\Sigma^2_{\max} + \sigma^2_{\max} \right)\log(T)$ which together with~\eqref{bound1inter} completes the proof.
  \end{proof}
  
  \section{Missing Proofs of Section \ref{sec:implications}}
  We first show the following general property of the variance for the ROM. This proposition will useful for showing the claims of this section. 
  \begin{proposition}\label{Var:ROM}
  For any $t \in [T]$, the variance of the ROM  with respect to $\cD_t$ satisfies
  \begin{align*}
  \E_{\xi \sim \mathcal D_t} \bigl[ \|\nabla f(x, \xi) - \nabla F^t(x)\|^2 \bigr]\leq \frac{T}{T-t+1}\sigma^2_1,
  \end{align*}
  for any $x \in \cX$.
  \end{proposition}
  \begin{proof}
    For any $x \in \cX$, since $ \nabla F^t(x) = \E_{\xi \sim \mathcal D_t}[\nabla f(x, \xi)]$, we have
  \begin{equation*}
    \E_{\xi \sim \mathcal D_t} \bigl[ \|\nabla f(x, \xi) - \nabla F^t(x)\|^2 \bigr]
    \leq \E_{\xi \sim \mathcal D_t} \bigl[ \|\nabla f(x, \xi) -   \nFf{1}(x)\|^2  
      \bigr] \, . 
  \end{equation*}
  Now, let  $\cT_t\subseteq [T]$ denote a subset of indices of gradients which remain to be selected in round $t$, and let $k_t \in \cT_{t-1} \setminus \cT_{t}$ be the index selected at round $t$.
 For any $x \in \cX$
  \begin{multline}\label{eq:sig_vs_sig1}
    \E_{\xi \sim \mathcal D_t} \Bigl[ \| \nabla f(x, \xi) -   \nFf{1}(x)  \|^2 \Bigr]
    = \frac{1}{T-t+1} \sum_{\xi \in \cT_t} \| \nabla f(x, \xi) -   \nFf{1}(x)  \|^2 \\
    \leq \frac{1}{T-t+1} \sum_{\xi \in [n]} \| \nabla f(x, \xi) -   \nFf{1}(x)  \|^2
    \leq \frac{T}{T-t+1} \sigma_1^2 \, , 
  \end{multline}  
  which is the claimed result.
  \end{proof}
  
    \subsection{Proof of Lemma \ref{lem:ROM:Lsmooth}}
  \label{proof:rom}
Note that in any case, $\tilde \sigma_1 \leq T \sigma_1$, and therefore $\log(\tilde \sigma_1 / \sigma_1)\leq \log(T)$. Thus Lemma~\ref{lem:ROM:Lsmooth} directly yields $\smash{\sigma_{[1:T]}^{(2)} \leq \sigma_1T \log(T)}$. This means in particular that the rate of OFTRL in the ROM is never more than a factor $\smash{\sqrt{\log T}}$ worse than the i.i.d. sampling with replacement rate of $\sigma_1 \sqrt T$; the next bound can often be much tighter.

\lemROMsmooth*
  \begin{proof}[Proof of Lemma~\ref{lem:ROM:Lsmooth}] 
   
  Let us begin with the adversarial variation.  We will show that deterministically (that is, for any order in which the losses are selected), for any $x \in \cX$, 
  \begin{equation*}
      \norm{\nabla \cF^t(x) - \nabla \cF^{t-1}(x)}^2 \leq \frac{4G^2}{(T-t+2)^2} \, .
  \end{equation*}
 
With the same notation as in Proposition \ref{Var:ROM}, recall that we denote by $\cT_t \subseteq [T]$ the support of $\mathcal D_t$ and $k_t = \cT_{t -1}\setminus \cT_{t}$.
We have $|\cT_t| = T-t+1$ and for any $x \in \cX$,
  \begin{align*}
    &\norm{\nabla \cF^t(x) - \nabla \cF^{t-1}(x)}^2 
    = \norm{\frac{1}{T-t+1} \sum_{s\in \cT_t} \nabla f_s(x) - \frac{1}{T-t+2} \sum_{s\in\cT_{t-1}} \nabla f_s(x)}^2\\
    &=  \norm{\frac{1}{(T-t+1)(T-t+2)} \sum_{s\in \cT_t} \nabla f_s(x) - \frac{1}{T-t+2} \nabla f_{k_t}(x)}^2\\
    &\leq \frac{2}{(T-t+2)^2} \norm{\frac{1}{(T-t+1)} \sum_{s\in \cT_t} \nabla f_s(x)}^2 +  \frac{2}{(T-t+2)^2}\norm{  \nabla f_{k_t}(x)}^2 \, .
  \end{align*}
  Thus, after maximising over $x \in \cX$, and taking expectations (note that the inequality holds almost surely) and summing over rounds $t\in [T]$, 
  \begin{align*}
    \Sigma_{[1:T]}^{(2)} = \E\Bigg[\sum_{t=1}^T \sup_{x\in\cX} \norm{\nabla \cF_{t}(x) - \nabla \cF_{t-1}(x)}^2 \Bigg]
    \leq  \sum_{t=1}^T  \frac{4G^2}{(T-t+2)^2} \leq 8 G^2.
  \end{align*}

  {\bfseries Variance. \quad} From Proposition~\ref{Var:ROM}, we know that
  \begin{align*}
  \sigma^2_t \leq \frac{T}{T-t+1}\sigma^2_1
  \end{align*}
   Moreover, one can see that
  \begin{multline}\label{eq:sig_vs_sigtilde}
    \E [ \sigma_t^2]
    \leq \E \Big[  \max_{x \in \cX} \, \E_{\xi \sim \mathcal D_t} \bigl[ \|\nabla f(x, \xi) -   \nFf{1}(x)\|^2 \bigr] \Big] \\
    \leq  \E \Bigl[ \E_{\xi \sim \mathcal D_t} \Bigl[ \max_{x \in \cX} \|\nabla f(x, \xi) -    \nFf{1}(x)\|^2 \Bigr] \Bigr]
    =  \tilde \sigma_1^2 \, .
  \end{multline}
   
  Let us introduce a threshold time step $\tau \in [T]$, of which we will set the value later. We upper bound $\E[\sigma_t^2]$ by \eqref{eq:sig_vs_sig1} for the rounds before $\tau$ and by \eqref{eq:sig_vs_sigtilde} for the other rounds:
  \begin{equation*}
    \E \bigg[\sum_{t=1}^T \sigma_t^2 \bigg]
    \leq \E \bigg[\sum_{t=1}^\tau  \sigma_t^2 \bigg]
    + \E \bigg[\sum_{t=\tau +1}^T \sigma_t^2 \bigg]
    \leq \sum_{t=1}^\tau \frac{T}{T-t+1} \sigma_1^2 + (T- \tau ) \tilde \sigma_1^2
  \end{equation*}
  Now using standard bounds on the harmonic series, 
  \begin{equation*}
    \sum_{t=1}^\tau \frac{1}{T - t + 1}
    = \sum_{n = T - \tau + 1}^T \frac{1}{n}
    \leq 1 + \log \frac{T}{T - \tau + 1} \, .
  \end{equation*}
  Therefore for any $\tau \in [T-1]$, we get
  \begin{equation}\label{eq:sigma_sum_any_tau}
    \E\bigg[\sum_{t=1}^T\sigma_t^2\bigg]
    \leq T \sigma_1^2 \bigg(1 + \log \frac{T}{T-\tau+1} \bigg)
    + (T- \tau) \tilde \sigma_1^2 \, . 
  \end{equation}
  We now conclude by setting the appropriate value for $\tau$. If $T \sigma_1^2 / \widetilde \sigma_1^2 \leq 2$, then $\log T \leq \log (2\widetilde \sigma_1^2 / \sigma_1^2)$, and taking $\tau =T$ gives a bound of $T \sigma_1^2 ( 1+ \log T) \leq T \sigma_1^2( 1 + \log (2 \widetilde \sigma_1^2 / \sigma_1^2))$, which is (better than) the claimed result.

  Otherwise, we take $\tau = T -  \lfloor  T  \sigma_1^2 /  \tilde \sigma_1^2 \rfloor$, then
  \begin{equation*}
    (T - \tau ) \tilde \sigma_1^2 
    =  \lfloor T \sigma_1^2 / \tilde\sigma_1^2 \rfloor \tilde\sigma_1^2
    \leq T \sigma_1^2 \, ,
  \end{equation*}
  and the argument of the logarithm can be bounded as
  \begin{equation*}
     \frac{T}{T - \tau + 1}
    \leq  \frac{T}{\lfloor T \sigma_1^2 / \tilde\sigma_1^2 \rfloor}
    \leq \frac{1}{ \sigma_1^2 / \tilde\sigma_1^2 - 1 / T}
    \leq \frac{2\tilde\sigma_1^2}{\sigma_1^2 } \, .
  \end{equation*}
  where we used the fact that $T \sigma_1^2 / \widetilde \sigma_1^2 > 2$. This yields the final bound
  \begin{equation*}
    \E\bigg[\sum_{t=1}^T\sigma_t^2\bigg]
    \leq T \sigma_1^2 \bigg( 1 + \log \frac{2 \tilde \sigma_1^2}{\sigma_1^2}\bigg)
     + T \sigma_1^2
    \leq T \sigma_1^2 \log \bigg(\frac{2e^2 \tilde \sigma_1^2}{\sigma_1^2} \bigg) \, .\qedhere
  \end{equation*}
  \end{proof}
   
\subsection{Proof of Corollary \ref{cor:ROM:sc:mpROM}}  
\label{cor:ROM:sc:proof}
\corROMsc*
  \begin{proof}[Proof of Corollary \ref{cor:ROM:sc:mpROM}]
  {\bfseries Single-pass ROM:} From Theorem~\ref{theorem:scRegretbound} we obtain (c.f. \eqref{sc:intermed:result})
\begin{align*}
\Exp{R_T(u)} &\leq \Exp{ \sum_{t=1}^T \frac{8}{\mu t} \|g_t  - \nabla F^t(x_t)\|^2 
      + \sum_{t=2}^T \frac{4}{\mu t} \| \nabla F^{t}(x_{t-1}) - \nabla F^{t-1}(x_t)\|^2     }\\
      &\qquad \qquad +GD + \frac{4L^2 D^2}{\mu} \log (1 + 16 \kappa) \, .
\end{align*}  
By Lemma~\ref{lem:ROM:Lsmooth}, we have
\begin{align*}
\Exp{ \sum_{t=2}^T \frac{4}{\mu t} \| \nabla F^{t}(x_{t-1}) - \nabla F^{t-1}(x_t)\|^2     } \leq 8G^2.
\end{align*}
Furthermore, recall that by Proposition \ref{Var:ROM}  $\E[\sigma_t^2] \leq T / (T -t +1) \sigma^2_1$.
\begin{align*}
\Exp{  \sum_{t=1}^T \frac{8}{\mu t} \|g_t  - \nabla F^t(x_t)\|^2} \leq \frac{8}{\mu} \sum_{t=1}^T \frac{T}{t(T-t+1)}\sigma_1^2 \leq \frac{8\sigma_1^2}{\mu}(2+ 2\log(T)).
\end{align*}
Indeed, using a standard bound on the harmonic series,
\begin{equation*}
  \sum_{t=1}^T \frac{T}{t(T-t+1)} = \sum_{t=1}^T \frac{T - t + 1 + t -1}{t(T-t+1)}
  \leq \sum_{t=1}^T \frac{1}{t} + \frac{1}{T-t+1}
  \leq 2 + 2\log T  \, . 
\end{equation*}
Combining these bounds gives the first part of the corollary. 

  \textbf{Multi-pass ROM: }  The critical term to upper bound, is the differences of the means whenever a pass ends and a new pass starts. Thus, for $P \in \NN$ passes, we need to control $\sup_{x \in \cX} \norm{\nFf{t}(x) - \nFf{t-1}(x)}^2$ for $t = j n + 1$, with $j \in [P]$. 

  Inside the $i$-th pass, for $k \in [n]$ we bound the $k$-th variation by
  \begin{equation*}
    \sup_{x\in\cX}\norm{\nFf{k}(x) - \nFf{k-1}(x)}^2
    \leq \frac{4G^2}{(n - k + 2)^2}\, ,
  \end{equation*}
  and we bound it by $G^2$ between the passes, so that 
  \begin{align*}
    &\frac{4}{\mu} \sum_{t=1}^T \frac{1}{t}   \sup_{x\in\cX}\norm{\nFf{t}(x) - \nFf{t-1}(x)}^2   \\ &\leq  \frac{4}{\mu} \sum_{i=1}^P\sum_{k=1}^n \frac{1}{(i-1)n + k}    \sup_{x\in\cX}\norm{\nFf{k}(x) - \nFf{k-1}(x)}^2   + \frac{4}{\mu} \sum_{i=1}^P  \frac{1}{in }    \sup_{x\in\cX}\norm{\nFf{1}(x) - \nFf{n}(x)}^2 \\ 
    &\leq  \frac{4}{\mu} \sum_{i=1}^P\sum_{k=1}^n \frac{1}{(i-1)n + k}  \frac{2G^2}{(n-k+2)^2} + \frac{4}{\mu} \sum_{i=1}^P  \frac{1}{in }    \sup_{x\in\cX}\norm{\nFf{1}(x) - \nFf{n}(x)}^2 \\
    &\leq \frac{16G^2}{\mu}\left(1+2 \frac{\log P}{n} \right) \, . \qedhere
  \end{align*}
  
\end{proof}

 \section{Batch-to-online Conversion}\label{appendixBatchToOnline}

 Consider the stochastic optimization problem $\min_{x\in\cX}\ExpD{\xi\sim\cD}{f(x,\xi)}$ and let $x^*$ denote a minimiser for this problem.
 Further, let $\cA$ be any first order stochastic optimization method with convergence guarantee $\ExpD{\xi \sim \cD}{f(x_t,\xi) - f(x^*,\xi)} \leq c(t)$.  As input $\cA$ takes an initial iterate $x_1$ and a sequence of i.i.d.~samples $\{f(\,\cdot\,,\xi_s)\}_{s\in[t]}  $. We let $\cA(x_1, \{f(\,\cdot\,,\xi_s) \}_{s \in[t]})$ denote the output $x_{t+1}$ of the stochastic optimization algorithm with respect to the given input. 
  Now consider an OCO  with $f(\,\cdot\,,\xi_1) , \dots  f(\,\cdot\,,\xi_T)$ and $\xi_1,\dots \xi_T$ are sampled i.i.d.~from a distribution.   
    \begin{algorithm}[htb]
 \label{batchtoonline}
\caption{Batch-to-online}
\textbf{Input: } Stochastic first order method $\cA$
\begin{algorithmic}[1]
\For{$t = 1,2,\ldots T$}
    \State{play $x_t$ and suffer loss $f(x_t,\xi_t)$}
    \State{restart $\cA$ and set $x_{t+1} = \cA(x_1, \{f(\,\cdot\,,\xi_s) \}_{s \in[t]})$}
\EndFor
\end{algorithmic}
 \end{algorithm}

 This batch-to-online conversion trivially achieves $\sum_{t=1}^T c(t)$ expected regret. However, with this conversion, some aspects of the stochastic convergence bound are lost. Consider for instance a convergence rate $c(t) = O(L D^2/t + D\sigma/\sqrt{t})$, from the the first-order stochastic approximation method in \cite{Ghadimi2013StochasticFA}  and the accelerated version $c(t) = O(L D^2/t^2 + D\sigma/\sqrt{t})$ \cite{ghadimi2012optimal,AccStochOptDM}. In both cases, the functions are assumed to satisfy \ref{A1}-\ref{A3}.  Batch-to-online conversion yields
 \begin{align*}
 \Exp{R_T(u)} \leq O(LD^2 \log T + D \sigma \sqrt{T})  \qquad \text{ and }\qquad  \Exp{R_T(u)} \leq O(LD^2 + D \sigma \sqrt{T}).
 \end{align*}
 The benefits of acceleration can be seen in the lower order terms. 
 Now using standard online-to-batch~\cite{CesaBianchi:2002} conversion in the way back gives the convergence bounds 
 \begin{align*}
  \Exp{f(x_T,\xi) - f(x^*,\xi)} \leq O\!\left( \frac{D \sigma}{ \sqrt{T}} + LD^2\frac{\log T}{T} \right) \text{ and } \Exp{f(x_T,\xi) - f(x^*,\xi)} \leq O\!\left(\frac{D \sigma}{ \sqrt{T}} + \frac{LD^2}{T} \right) \, ,
 \end{align*}
 In the case of accelerated stochastic approximation, the benefits of acceleration are inevitably lost through batch-to-online and online-to-batch conversion.

\section{Additional Examples for Intermediate Cases} 
\label{appendix:examples:intermediate}
 We provide regret bounds for intermediate cases not discussed in the main body of the paper, namely the cases when the adversary selects slowly shifting distributions and when the adversary switches rarely between distributions.
 \paragraph{Distribution shift: } In this example, the SEA picks $\cD_{t}$ and $\cD_{t-1}$, such that $  \nFf{t}(x) $ is close to the mean of the previous distribution  gradient $ \nFf{{t-1}}(x)$. We shall consider two kinds of distribution shifts. Firstly, when the means are close on average, that is, when $(1/T)\sum_{t=1}^T \sup_{x\in\cX} \tnorm{ \nFf{{t}}(x) -  \nFf{{t-1}}(x)}^2 \leq \epsilon$, secondly, when this holds for each iteration $t$, i.e., $\sup_{x\in\cX} \tnorm{ \nFf{t}(x) -  \nFf{t-1}(x)}^2 \leq \epsilon$. We refer to the former as the \emph{average distribution shift} case, and to the latter as the \emph{bounded distribution shift} case. 
 
For strongly convex functions, Theorem~\ref{theorem:scRegretbound} directly yields the regret bound
  \begin{align*}
 \Exp{R_T(u)}  \leq O\!\left(  \frac{1}{\mu}(\sigma^2_{\max}  +\epsilon)\log T+ D^2 L \kappa \log \kappa  \right).
 \end{align*}
 For the considerably weaker assumption of an average distribution shift, we have   
 \begin{align*}
  \sum_{t=1}^T \frac{1}{t\mu}\sup_{x\in\cX}\norm{\nFf{t}(x) - \nFf{{t-1}}(x)}^2 \leq \Sigma_{[1:T]}^{(2)} \sqrt{ \sum_{t=1}^T \frac{1}{t^2 \mu^2}}     \leq \frac{4}{\mu} T\epsilon.
 \end{align*}
   To obtain the first inequality, we have used the Cauchy-Schwarz inequality together with the fact that $\sqrt{a+b} \leq \sqrt{a} + \sqrt{b}$, and the second inequality follows directly from the definition of the averaged distribution shift. Now suppose $\epsilon \leq 1/T$, then we obtain the following regret bound in case of average distribution shift.
  \begin{equation*}
    \Exp{R_T(u)}  \leq O \! \left(  \frac{\sigma^2_{\max}}{\mu}  \log T+ \frac{ 1}{\mu}   + D^2 L \kappa \log \kappa  \right).
  \end{equation*} 
 Since 
 $\Sigma^{(2)}_{[1:T]} \leq T\epsilon$ for the average distribution shift, for convex and smooth functions, Theorem~\ref{thm:VarStepsize} entails that
\begin{align*}
\Exp{R_T(u) }  \leq O\left( D(\sigma_{\max} + \sqrt{\epsilon})\sqrt{T  }  +   DG+ LD^2 \right).
\end{align*}
 
 \paragraph{Distribution switch: }  SEA switches $c$ times between distributions  
  $\cD_1, \dots, \cD_c \in \dD$. These switches can happen at any round  and the learner does not know when a switch occurs. In this case, we can upper bound $\smash{\Sigma^{(2)}_{[1:T]} \leq \Sigma^2_{\max} c}$. 
  Thus, for strongly convex functions Theorem~\ref{theorem:scRegretbound} directly yields
 \begin{equation*}
 \Exp{R_T(u)}  \leq O\left(  \frac{1}{\mu} \left( \sigma^2_{\max} \log T +\Sigma^2_{\max} \log c \right) + D^2 L \kappa \log \kappa  \right).
 \end{equation*}
 And for convex smooth functions Theorem \ref{thm:VarStepsize} gives
\begin{equation*}
\Exp{R_T(u)} \leq O\!\left( D\sigma_{\max}\sqrt{T} + D\Sigma_{\max} \sqrt{c}  +  DG + LD^2 \right).
\end{equation*}

\end{document}